\pdfoutput=1
\def\year{2018}\relax
\documentclass[letterpaper]{article} 
\usepackage{aaai18}  
\usepackage{times}  
\usepackage{helvet}  
\usepackage{courier}  
\usepackage{url}  
\usepackage{graphicx}  
\usepackage{subcaption}
\usepackage{sidecap}
\usepackage{wrapfig}
\usepackage{tikz}
\usepackage{varwidth}
\usepackage{amsmath, bbm, amssymb, amsfonts, amsthm}
\usepackage{algorithm}
\usepackage{algpseudocode}
\usepackage[english]{babel}
\usepackage{marvosym}
\usepackage{epstopdf}
\newtheorem{theorem}{Theorem}[section]
\newtheorem{lemma}[theorem]{Lemma}

\newtheorem{num_proof}[theorem]{Proof}
\newtheorem{definition}[theorem]{Definition}

\graphicspath{ {./figure/} }

\begin{document}
%
\title{Towards Training Probabilistic Topic Models on \\ Neuromorphic Multi-chip Systems}
\author{Zihao Xiao,\ Jianfei Chen,\ Jun Zhu\thanks{corresponding author.}\\
Dept. of Comp. Sci. \& Tech., TNList Lab, State Key Lab for Intell. Tech. \& Sys.\\
Center for Bio-Inspired Computing Research, Tsinghua University, Beijing, 100084, China\\
\{xiaozh15,~chenjian14\}@mails.tsinghua.edu.cn,~ dcszj@tsinghua.edu.cn
}
\maketitle
\begin{abstract}
Probabilistic topic models are popular unsupervised learning methods, including probabilistic latent semantic indexing (pLSI) and latent Dirichlet allocation (LDA).
By now, their training is implemented on general purpose computers (GPCs), which are flexible in programming but energy-consuming.
Towards low-energy implementations, this paper investigates their training on an emerging hardware technology called the neuromorphic multi-chip systems (NMSs).
NMSs are very effective for a family of algorithms called spiking neural networks (SNNs).
We present three SNNs to train topic models.
The first SNN is a batch algorithm combining the conventional collapsed Gibbs sampling (CGS) algorithm and an inference SNN to train LDA.
The other two SNNs are online algorithms targeting at both energy- and storage-limited environments.
The two online algorithms are equivalent with training LDA by using maximum-a-posterior estimation and maximizing the semi-collapsed likelihood, respectively.
They use novel, tailored ordinary differential equations for stochastic optimization.
We simulate the new algorithms and show that they are comparable with the GPC algorithms, while being suitable for NMS implementation.
We also propose an extension to train pLSI and a method to prune the network to obey the limited fan-in of some NMSs.
\end{abstract}

\section{Introduction}\label{sec.Introduction}
Topic models have been widely used to discover latent semantic structures from a large corpus of text documents or images in a bag-of-words format~\cite{Sivic2005}.
The most popular models are probabilistic Latent Semantic Indexing (pLSI)~\cite{Hofmann1999} and its Bayesian formulation of Latent Dirichlet Allocation (LDA)~\cite{Blei2003}.
On mobile applications, LDA is a robust Bayesian model that is suitable to learn from small, noisy data, e.g., images, texts and context logs~\cite{Bao2012}.
Moreover, some researchers have reported that learning LDA with a hybrid architecture consisting of mobile devices and servers reduces costs on the server end, as well as the response times on the mobile end~\cite{Robinson2015}.
On servers, much recent progress has been made on developing efficient algorithms to learn a large number of topics on massive-scale datasets~\cite{Wang2014,Chen2016,Yuan2015}.
However, all these algorithms are implemented on general purpose computers (GPCs), which are powerful in computing and flexible in programming but energy-consuming.

Neuromorphic multi-chip systems (NMSs) represent an emerging hardware technology for low-energy implementations of spiking neural networks (SNNs).
SNNs were originally studied to understand the computation in the brain, where a neuron communicates with other neurons via voltage spikes.
Different from GPCs, NMSs have some special designs, making it highly nontrivial to implement an ordinary leaning algorithm.
First, there are dedicated computing units calculating a weighted sum of the inputs and triggering spikes accordingly.
Second, NMSs use distributed memory~\cite{Merolla2014}, and the on-chip memory is typically limited (e.g., 52MB for TruthNorth~\cite{Merolla2014} and 128MB for SpiNNaker~\cite{Furber2014}).
When a SNN is large, its neurons and model parameters (e.g., synaptic weights) must reside on several chips.
As the inter-chip communication is only efficient for spikes by using the AER protocol~\cite{Mahowald1994}, parameter communication is either impossible or inefficient.
To extend the storage size, a NMS can be integrated with an external memory~\cite{Connor2013} to store some infrequently accessed data.
But in many scenarios where low-energy computing is required, the external memory can also be limited (e.g., on a mobile phone).
Thus, it is important to consider the limited memory issue when implementing new models.

For topic models, existing learning algorithms do not satisfy the basic computation and communication designs of NMSs; thus cannot be directly implemented on NMSs.
For the popular collapsed Gibbs sampling (CGS)~\cite{Griffiths2004}, the update of a model parameter (i.e., a count) only depends on its latest value, without needing parameter communication.
But the sampling operation in CGS is not implemented by SNN dynamics compatible with NMSs.
When the limited external memory is considered, CGS is undesirable because it stores the whole corpus and all topic assignments.
Moreover, though the existing online or stochastic algorithms are memory-efficient, they require intensive parameter communications, i.e., the update of one parameter depends on the exact value of another parameter.
Specifically, the variational inference (VI) methods~\cite{Blei2003,Hoffman2013,Broderick2013} have mutual dependency between the local parameters and global ones during learning;
and the stochastic gradient MCMC (SGMCMC)~\cite{Patterson2013} also has strong coupling between different dimensions when computing the log-likelihood gradient as it has a normalization term.

In this paper, we aim to fill up the gap between exiting algorithms for topic models (particularly LDA and pLSI) and the special requirements of NMSs by designing novel algorithms suitable for NMS implementation.
We draw inspirations from the work on the NMS implementation of multinomial mixtures~\cite{Nessler2013}, where an online SNN algorithm is designed to meet all the design requirements of NMSs.
We significantly extend this work by presenting three new SNNs to learn topic models, which are much more complicated than mixture models by introducing more random variables and Bayesian priors.

The proposed SNNs introduce an additional document layer and additional connections to represent the documents and their topic-mixing proportions.
The first SNN is a SNN implementation of CGS by leveraging its nice locality property.
The second SNN applies the theory behind the online SNN algorithm in \cite{Nessler2013} (i.e., ordinary differential equation (ODE) and its stochastic approximation) to solve a conceived optimization problem that is equivalent to the Maximum-a-Posteriori (MAP) problem of LDA.
The third SNN solves a conceived optimization problem that is equivalent to maximizing the semi-collapsed likelihood of LDA.
It is a hybrid of the former two networks that has a CGS component to sample the local latent variables, and an optimization component to update the global parameters.
Empirical results show that our online SNN algorithms are comparable with existing GPC algorithms while they have the advantage of being suitable for NMS implementation.

In the appendix, we propose a SNN implementation of training pLSI as a special case of LDA.
And we also propose a network pruning scheme to satisfy the limited fan-in in some NMS designs, e.g., TruthNorth~\cite{Merolla2014}.

\section{Preliminary of Topic Models}\label{sec.Preliminary}
For clarity, we focus on Latent Dirichlet Allocation (LDA)~\cite{Blei2003}.
All our techniques can be applied to pLSI, a special case of LDA.
We defer the details to Appendix.

\subsection{Latent Dirichlet Allocation}\label{sec.Latent Dirichlet Allocation}
Consider a corpus $\mathbf{W}\triangleq\{\mathbf{w}_d\}_{d=1}^D$ with $D$ documents and $V$ unique words in its vocabulary.
$\mathbf{w}_d\triangleq (w_{d1},...,w_{dN_d})$ denotes document $d$ with $N_d$ words and $w\in \{1,...,V\}$ denotes the occurrence of one word.
LDA assumes a generative process for the corpus $\mathbf{W}$ as follows: 
\begin{enumerate}
    \item[]for each document $d=1,...,D$,
        \begin{enumerate}%
            \item[] draw a topic mixing proportion $\boldsymbol\theta_d\sim\text{Dir}(\boldsymbol\lambda)$;
            \item[] for each position in the document, $i=1,...,N_d$,
                \begin{enumerate}
                    \item[] draw a topic assignment $z_{di}\sim\text{Multi}(\boldsymbol\theta_d)$;
                    \item[] draw a word $w_{di}\sim\text{Multi}(\boldsymbol\phi_{z_{di}})$,
                \end{enumerate}
        \end{enumerate}
\end{enumerate}
where $\boldsymbol\theta_d$ is a $K$-dimensional topic-mixing proportion vector of document $d$; $\boldsymbol\lambda$ is the hyper-parameter of the Dirichlet prior; $K$ is the pre-specified number of topics; $\boldsymbol\phi_k$ is a $V$-dimensional topic distribution vector for topic $k$; and $\text{Multi}(\cdot)$ and $\text{Dir}(\cdot)$ denotes the Multinomial distribution and the Dirichlet distribution respectively.
We will use $\mathbf{z}_d\triangleq(z_{d1},...,z_{dN_d})$ to denote the collection of topic assignments for document $d$, and use the notations $\mathbf{Z}\triangleq\{\mathbf{z}_d\}_{d=1}^D$, $\boldsymbol\Phi \triangleq \{\boldsymbol\phi_k\}_{k=1}^K, \boldsymbol\Theta \triangleq \{\boldsymbol\theta_d\}_{d=1}^D$ in the sequel.

In this paper, we focus on two parameter estimation methods for LDA: the Maximum-a-Posteriori (MAP) estimation and the Maximum-likelihood (ML) estimation on the semi-collapsed distribution.
The MAP training has proven effective in \cite{Asuncion2009,Chen2016}\footnote{Strictly speaking, they use a smoothed LDA described next.}.
While MAP is a point-estimate, using the semi-collapsed distribution provides a better parameter estimation on $\boldsymbol\Phi$ by integrating out the latent variable $\boldsymbol\Theta$~\cite{Patterson2013,Griffiths2004}.
In the following, we extend the MAP and ML problems from the batch setting to the online setting, which is explored in this paper.

{\bf MAP:} The MAP problem in the batch setting is to maximize $\log p(\boldsymbol\Theta ; \mathbf{W}, \boldsymbol\Phi, \boldsymbol\lambda)$, which is a summation of the log-likelihood $\log p(\mathbf{W};\boldsymbol\Phi,\boldsymbol\Theta)$ and the log-prior $ \log p(\boldsymbol\Theta|\boldsymbol\lambda)$.
According to the i.i.d. assumptions, the log-likelihood can be re-written as
{\small
\begin{align}
    \log p(\mathbf{W};\boldsymbol\Phi,\boldsymbol\Theta)
    &= \sum_{d=1}^D\sum_{w=1}^V N_{wd}\log p(w|d;\boldsymbol\Phi, \boldsymbol\Theta),\nonumber
\end{align}
}%
where $N_{wd}$ is the number of times that a word $w$ occurs in document $d$.
Let $\pi(w,\!d)$ denote the empirical distribution of the co-occurrence of word $w$ and document $d$, we have
{\small
\begin{align}
    \log p(\mathbf{W};\boldsymbol\Phi,\boldsymbol\Theta)  &= N\sum_{d=1}^D\sum_{w=1}^V \frac{N_{wd}}{N}\log p(w|d;\boldsymbol\Phi, \boldsymbol\Theta)\nonumber\\
    &=N\mathbb{E}_{\pi(w,d)} \big[\log p(w|d;\boldsymbol\Phi,\boldsymbol\Theta) \big],\label{equ.L stochastic form}
\end{align}
}%
where $N=\sum_{d=1}^{D}\sum_{w=1}^V N_{wd}$ is the total number of tokens in the corpus. Furthermore, as the prior distributions of $\boldsymbol\theta_d$'s are mutually independent, the log-prior term can be re-written as
{\small
\begin{align}
     \log p(\boldsymbol\Theta|\boldsymbol\lambda)\!=\!\sum_{d=1}^D\log p(\boldsymbol\theta_{d}|\boldsymbol\lambda)
    \!=\!N\mathbb{E}_{\pi(w,d)} \!\!\big[ \frac{1}{N_d}\log p(\boldsymbol\theta_d|\boldsymbol\lambda) \big]. \nonumber 
\end{align}
}%
With the above derivations, the MAP problem becomes
{\small
\begin{align}
\max_{\boldsymbol\Phi,\boldsymbol\Theta}~ \mathbb{E}_{\pi(w,d)}\Big[\log p(w|d;\boldsymbol\Phi,\boldsymbol\Theta) + \frac{1}{N_d}\log p(\boldsymbol\theta_d|\boldsymbol\lambda)\Big].
    \label{equ.MAP stochastic problem}
\end{align}
}%
In the new formulation of Eq.~(\ref{equ.MAP stochastic problem}), $\pi(w,d)$ is not limited to the empirical distribution of a fixed corpus.
Instead, it can also be the unknown environment distribution representing the underlying probability of the co-occurrence of word $w$ and document $d$.
As long as we can draw samples (e.g., data comes in a stream), an unbiased estimate of the objective can be constructed.

{\bf ML on the semi-collapsed distribution:} Leveraging the semi-collapsed distribution can attain a better parameter estimation on $\boldsymbol\Phi$.
Specifically, we consider maximizing the semi-collapsed likelihood $p(\mathbf{W};\boldsymbol\Phi,\boldsymbol\lambda) = \int p(\mathbf{W}|\boldsymbol\Theta;\boldsymbol\Phi)p(\boldsymbol\Theta;\boldsymbol\lambda)d\boldsymbol\Theta$,
where $\boldsymbol\Theta$ is integrated out.

Using the evidence lower bound (ELBO), the logarithm of this semi-collapsed likelihood can be re-written as
{\small
\begin{align}
    \log p(\mathbf{W}|\boldsymbol\Phi,\boldsymbol\lambda) = \mathbb{E}_{p(\mathbf{Z}|\mathbf{W};\boldsymbol\Phi,\boldsymbol\lambda)}\log p(\mathbf{W}|\mathbf{Z};\boldsymbol\Phi) + C\label{equ.semi ELBO}
\end{align}
}%
where $C$ denotes a constant w.r.t. $\boldsymbol\Phi$.
Detailed derivation of this equality is shown in Appendix~\ref{appendix.elbo}.
This transformation allows us to leverage the many i.i.d. structures in LDA:
{\small
\begin{align*}
    \log p(\mathbf{W}|\boldsymbol\Phi,\boldsymbol\lambda) 
    =& D\mathbb{E}_{\pi(d)}\mathbb{E}_{p(\mathbf{z}_d|\mathbf{w}_d;\boldsymbol\Phi,\boldsymbol\lambda)}\log p(\mathbf{w}_d|\mathbf{z}_d;\boldsymbol\Phi) + C,
\end{align*}
\begin{align*}
    =D\mathbb{E}_{\pi(d)}\sum_{w=1}^V\sum_{z=1}^K\mathbb{E}_{p(\mathbf{z}_d|\mathbf{w}_d;\boldsymbol\Phi,\boldsymbol\lambda)}[C_{d,z,w}]\log p(w|z;\boldsymbol\Phi) + C,
\end{align*}
}%
where $\pi(d)$ denotes the empirical distribution of documents; $C_{d,z,w}$ denotes the number of words $w$ assigned topic $z$ in document $d$.
The first equality is from the i.i.d. documents, and the second is from the i.i.d. words given topics.
Overall, the ML problem on the semi-collapsed likelihood is
{\small
\begin{align}
    \max_{\boldsymbol\Phi}\mathbb{E}_{\pi(d)}\mathbb{E}_{p(\mathbf{z}_d|\mathbf{w}_d;\boldsymbol\Phi,\boldsymbol\lambda)}[C_{d,z,w}]\log p(w|z;\boldsymbol\Phi),\label{equ.semi stochastic problem}
\end{align}
}%
where $\pi(d)$ can be extended to represent the unknown environment distribution where only its samples are accessible.

We use Gibbs sampling to approximately infer the semi-collapsed posterior distribution $p(\mathbf{z}_d|\mathbf{w}_d;\boldsymbol\Phi,\boldsymbol\lambda)$ in Eq.~(\ref{equ.semi stochastic problem}).
Gibbs sampling iteratively performs the following steps until convergence~\cite{Patterson2013}:
\begin{enumerate}
\item Uniformly sample a token $w_{di}$ from document $d$
\item Sample its topic assignment $z_{di}$ from the local conditional distribution:
{\small
\begin{align}
p(z_{di}=z|\mathbf{z}_d^{\neg di}, \mathbf{w}_d;\boldsymbol\Phi,\boldsymbol\lambda)
           \propto \phi_{zw} \cdot (C_{z,d}^{\neg di} + \lambda_z),\label{equ.semi cgs update rule}
\end{align}
}%
\end{enumerate}
where $C_{z,d}$ is the count of times that topic $z$ is assigned to any token in document $d$; the superscript $\neg di$ means that the $i$ position in document $d$ is eliminated from the counts or variable collections.
In this paper, we call this Gibbs sampling as semi-CGS.

Note that the corpus size $D$ is not required in problems (\ref{equ.MAP stochastic problem}) and (\ref{equ.semi stochastic problem}).
This is different from the stochastic methods \cite{Hoffman2013,Patterson2013} which need to know $D$.
As a result, formulation~(\ref{equ.MAP stochastic problem}) can be used in streaming data settings, as explored in this paper.

The reason for omitting $D$ is that we don't define priors on the global parameter $\boldsymbol\Phi$.
This is the case that our online SNN algorithms, under the hardware constraints, can only deal with at present.
Although no prior for $\boldsymbol\Phi$ might slightly harm the generalization performance in comparison with the smoothed LDA introduced below, this vanilla LDA model is originally used in~\cite{Blei2003}.

\subsection{Smoothed LDA and collapsed Gibbs sampling}\label{sec.CGS}
The smoothed LDA further defines a Dirichlet prior for the topic distributions $\boldsymbol\phi_k$ upon a LDA to improve the generalization performance:
{\small
\begin{align}
    \boldsymbol\phi_k\sim\text{Dir}(\boldsymbol\varphi), ~~ k = 1,...,K.\nonumber
\end{align}
}%
A popular algorithm to train the smoothed LDA is Collapsed Gibbs Sampling (CGS)~\cite{Griffiths2004}.
CGS is a batch algorithm that uses Gibbs sampling to sample from the collapsed posterior distribution $p(\mathbf{Z}|\mathbf{W},\boldsymbol\lambda,\boldsymbol\varphi)$  and the inference can have a high accuracy.
CGS iteratively performs the following steps until convergence:
\begin{enumerate}
   \item Uniformly sample a token $w_{di}$ from the corpus.
   \item Sample its topic assignment $ z_{di}$ from the local conditional distribution:
\end{enumerate}
{\small
\begin{align}
p(z_{di}\!=\!z|\mathbf{Z}^{\neg di}, \mathbf{W};\boldsymbol\lambda,\boldsymbol\varphi)
           \!\propto\! \frac{C_{w_{di},z}^{\neg di} \!+\! \varphi_{w_{di}}}{C_{\cdot,z}^{\neg di} \!+\! \bar{\varphi}} \!\cdot\! (C_{z,d}^{\neg di} \!+\! \lambda_z),\label{equ.cgs update rule}
\end{align}
}%
where $C_{w,z}$ is the count of times that word $w$ is assigned to topic $z$, $C_{\cdot,z} = \sum_{w=1}^VC_{w,z}$; $\bar{\varphi} = \sum_{w=1}^V\varphi_w$.

\section{The SNN Algorithms}\label{sec.The SNN algorithms}
We now introduce our SNN algorithms to train LDA.
First, the network architecture and activation method are introduced to encode the word, the topic and the document.
Then we present a SNN implementation of CGS, and two online SNN algorithms to solve problem (\ref{equ.MAP stochastic problem}) and (\ref{equ.semi stochastic problem}) respectively.

\subsection{Network architecture}\label{sec.SNN representation}
The proposed neural network architecture is shown in Fig.~\ref{fig.SNN architecture}.
There are two observable layers (the word and the document) and one latent layer (the topic).
The three layers use one-hot representations to encode a word, a document and a topic respectively; see Tab.~\ref{tlb.Parameter relation}(a).
The latent layer is fully connected by the observable layers, where $\mathbf{M} \triangleq \{ \mathbf{M}^{\alpha}, \mathbf{M}^{\beta} \}$ denotes all synaptic weights.
In the SpikeCGS algorithm (see Tab.~\ref{tlb.Parameter relation}(b)), there are self-excitations $\mathbf{b}$ in the latent layer.

\begin{figure}[!t]
    \centering
    \resizebox{0.35\textwidth}{!}{
    \begin{tikzpicture}[node distance=\nodedist]
        \def\nodesize{15pt};

        \tikzstyle{every pin edge}=[<-,shorten <=2pt]
        \tikzstyle{neuron}=[circle, draw=black!100, line width=0.5mm, inner sep=0pt, minimum size=\nodesize]
        \tikzstyle{dot}=[circle, draw=black!100, fill=black!100, line width=0.5pt, inner sep=0.5pt, minimum size=0.1mm]
        \tikzstyle{annot} = [text width=10em, text centered]

        \foreach \name / \x in {0,1,4,5}
            \path[xshift=-2cm]
                node[neuron] (w-\name) at ({-\x * \nodesize * 1}, 0) {};
        \foreach \name / \x in {3.1,2.5,1.9}
            \path[xshift=-2cm]
                node[dot] (w-\name) at ({-\x * \nodesize * 1}, 0) {};
        \draw[xshift=-2cm, line width = 0.5mm] ({0.5 * \nodesize}, {0.5 * \nodesize + 0.1}) rectangle ({-5.5 * \nodesize}, {-0.5 * \nodesize - 0.1}) node[right = {3 * \nodesize}, below = {0.1*\nodesize}]{word layer  $\mathbf{x}^{\alpha}$};

        \foreach \name / \x in {0,3}
            \path[xshift=-0.5cm]
                node[neuron] (d-\name) at ({\x * \nodesize * 1}, 0) {};
        \foreach \name / \x in {0.9,1.5,2.1}
            \path[xshift=-0.5cm]
                node[dot] (d-\name) at ({\x * \nodesize * 1}, 0) {};
        \draw[xshift=-0.5cm, line width = 0.5mm] ({-0.5 * \nodesize}, {0.5 * \nodesize + 0.1}) rectangle ({3.5 * \nodesize}, {-0.5 * \nodesize - 0.1}) node[left = {2*\nodesize}, below = {0.1*\nodesize}]{document layer  $\mathbf{x}^{\beta}$};

        \foreach \name / \x in {-0.5,2.5}
            \path[xshift=-2.25cm]
                node[neuron] (t-\name) at ({\x * \nodesize * 1}, {3 * \nodesize}) {};
        \foreach \name / \x in {0.5,1,1.5}
            \path[xshift=-2.25cm]
                node[dot] (t-\name) at ({\x * \nodesize * 1}, {3 * \nodesize}) {};
        \draw[xshift=-2.25cm, yshift={3 * \nodesize}, line width = 0.5mm] ({-1 * \nodesize}, {0.5 * \nodesize + 0.1}) rectangle ({3 * \nodesize}, {-0.5 * \nodesize - 0.1}) node[left = {2 * \nodesize}, above = \nodesize] {topic layer  $\mathbf{h}$} node[right = {0.5 * \nodesize}, above = 0 * \nodesize] {$\mathbf{b}$};

        \draw[line width =1mm,>=latex,->,draw=black!100] ({-2cm - 3 * \nodesize}, {\nodesize * 0.75}) -- ({-2.25cm + \nodesize * 0.7}, {2.25 * \nodesize}) node[left = 1.5cm, below = 0.00cm, scale = 1.2] {$\mathbf{M}^{\alpha}$};
        \draw[line width =1mm,>=latex,->,draw=black!100] ({-0.5cm + 1 * \nodesize}, {\nodesize * 0.75}) -- ({-2.25cm + \nodesize * 1.5}, {2.25 * \nodesize}) node[right = 1.6cm, below = 0.00cm, scale = 1.2] {$\mathbf{M}^{\beta}$};
    \end{tikzpicture}
    }
    \caption{SNN architecture}
    \label{fig.SNN architecture}
\end{figure}
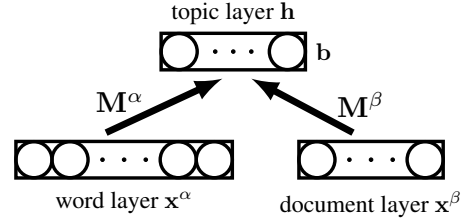

\begin{table}[t!]
\begin{center}
\resizebox{0.45\textwidth}{!}{%
    \begin{tabular}{ |c| c |c| }\hline
        Topic model & Network & Relation\\
        variables &variables & \\\hline
        $z$ & $\mathbf{h}$ & $z=z^*\leftrightarrow h_{z^{'}}=\mathbb{I}(z^{'}=z^*),\forall z^{'}$\\
        $w$ & $\mathbf{x}^{\alpha}$ & $w=w^*\leftrightarrow x_{w^{'}}^{\alpha}=\mathbb{I}(w^{'}=w^*),\forall w^{'}$\\
        $d$ & $\mathbf{x}^{\beta}$ & $d=d^*\leftrightarrow x_{d^{'}}^{\beta}=\mathbb{I}(d^{'}=d^*),\forall d^{'}$\\\hline
        \multicolumn{3}{c}{(a) for all SNNs}\\\hline
        $C_{w,z}^{\neg di}+\varphi_w$ & $M_{zw}^{\alpha}$ & $C_{w,z}^{\neg di}+\varphi_w = \exp M_{zw}^{\alpha}$\\
        $C_{z,d}^{\neg di}+\lambda_z$ & $M_{zd}^{\beta}$ & $C_{z,d}^{\neg di}+\lambda_z = \exp M_{zd}^{\beta}$\\
        $C_{\cdot,z}^{\neg di}+\bar{\varphi}$ & $b_z$ & $C_{\cdot,z}^{\neg di}+\bar{\varphi} = \exp b_z$\\\hline
        \multicolumn{3}{c}{(b) for SpikeCGS (when resampling)}\\\hline
        $\phi_{zw}$ & $M_{zw}^{\alpha}$ & $\phi_{zw} \propto \exp M_{zw}^{\alpha}$\\
        $\theta_{dz}$ & $M_{zd}^{\beta}$ & $\theta_{dz} \propto \exp M_{zd}^{\beta}$\\\hline
        \multicolumn{3}{c}{(c) for SpikeLDA}\\\hline
        $\phi_{zw}$ & $M_{zw}^{\alpha}$ & $\phi_{zw} \propto \exp M_{zw}^{\alpha}$\\
        $C_{z,d}^{\neg di}+\lambda_z$ & $M_{zd}^{\beta}$ & $C_{z,d}^{\neg di}+\lambda_z = \exp M_{zd}^{\beta}$\\\hline
        \multicolumn{3}{c}{(d) for semi-SpikeLDA}
    \end{tabular}%
    }
    \end{center}
    \caption{Re-parametrization, where $\mathbb{I(\cdot)}$ represents the indicator function.}
    \label{tlb.Parameter relation}
\end{table}

\subsection{Activation method}
Suppose a token $(w,d)$ is observed.
The corresponding neurons in the observable layers are clamped and excited following the one-hot encoding (Tab.~\ref{tlb.Parameter relation}(a)).
Then the neurons in the latent layer trigger their spikes following Poisson processes asynchronously until any latent neuron succeeds in firing a spike.
Specifically, the $z$th latent neuron calculates the input from the precedent layer $u_z\!=\!\mathbf{M}_{z\cdot}^{\alpha} \cdot \mathbf{x}^{\alpha}\!+\! \mathbf{M}_{z\cdot}^{\beta}\cdot\mathbf{x}^{\beta}$\footnote{When there is self-excitation, $u_z = \mathbf{M}_{z\cdot}^{\alpha} \cdot \mathbf{x}^{\alpha}  + \mathbf{M}_{z\cdot}^{\beta}\cdot\mathbf{x}^{\beta} - b_z$.}, and fires spikes following an inhomogeneous Poisson process of rate $\exp(u_z)$.
If the $z$th latent neuron fires a spike, then $h_z=1$ and topic $z$ is sampled; otherwise $h_z=0$.
This uses the one-hot encoding as well (Tab.~\ref{tlb.Parameter relation}(a)).
Alg.~\ref{alg.E-step}
summarizes this activation dynamics.
This activation dynamics was previously used in~\cite{Nessler2013}, and they prove that the sample, i.e. the output $\hat{z}$ in Alg.~\ref{alg.E-step}, is a sample from a softmax distribution $p(z|\cdot)\!\propto\!\exp(u_z)$.

\begin{algorithm}[t!]
    \caption{Sample the topic assignment $\hat{z}$ for a token $(w,d)$}
    \begin{algorithmic}[1]
        \Function{InferenceDynamics}{$w,d$}
        \State Set $x_{w'}^{\alpha} = \mathbb{I}(w=w'),\forall w';x_{d'}^{\beta} = \mathbb{I}(d=d'),\forall d'$ \label{alg.line.one hot}
        \While{no latent neuron trigger a spike yet}
            \For{latent neuron $z\!=\!1,...K$, asynchronously}
                \State $u_z = \mathbf{M}_{z\cdot}^{\alpha} \cdot \mathbf{x}^{\alpha}  + \mathbf{M}_{z\cdot}^{\beta}\cdot\mathbf{x}^{\beta}$
                \State fire spikes using the rate $\exp(u_z)$
            \EndFor
        \EndWhile
        \State \Return index of the fired neuron, $\hat{z}$
        \EndFunction
    \end{algorithmic}
    \label{alg.E-step}
\end{algorithm}

\subsection{The SNN implementation of CGS}\label{sec.cgs snn}
CGS has the good locality property that the update of a count only depends on its latest value, without intensive parameter communication.
This satisfies the communication constraints of NMSs.
However, it does not sample the topic assignment, i.e. Eq.~(\ref{equ.cgs update rule}), using SNN dynamics.
We solve this problem by using the SNN dynamics Alg.~\ref{alg.E-step}.

Suppose a token $w_{di}$ is sampled in an iteration of CGS and now
its topic assignment should be resampled from Eq.~(\ref{equ.cgs update rule}).
If we represent the word $w_{di}$, its document $d$ and its topic assignment using the spiking neurons via Tab.~\ref{tlb.Parameter relation}(a), the conditional distribution to sample from is equivalent with
{\small
\begin{align*}
p(h_z\!=\!1|\cdot)
           \!\propto\! \frac{C_{w_{di},z}^{\neg di} \!+\! \varphi_{w_{di}}}{C_{\cdot,z}^{\neg di} \!+\! \bar{\varphi}} \!\cdot\! (C_{z,d}^{\neg di} \!+\! \lambda_z),~~~\forall z,
\end{align*}
}%
on the SNN.
If the network parameters relate to the counts via Tab.~\ref{tlb.Parameter relation}(b), the conditional distribution becomes
{\small
\begin{align}
    p(h_z=1|\cdot) \propto \exp(M_{zw_{di}}^{\alpha} + M_{zd}^{\beta} - b_z),~~~\forall z,\label{equ.cgs snn inference generalization}
\end{align}
}%
which can be implemented by the activation method Alg.~\ref{alg.E-step}.
Because $M_{zw}$ and $M_{zd}$ have one-to-one correspondence with the counts, the locality property of CGS is inherited.
The SNN implementation of CGS is called SpikeCGS and is summarized in Alg.~\ref{alg.SpikeCGS}.

Using the relations between the network parameters and the counts in Tab.~\ref{tlb.Parameter relation}(b), one can easily prove that line 4 corresponds to eliminating the current token from the counts, and line 6 corresponds to updating the counts according to the resampling result, similar as a regular implementation of CGS.
Obviously, the updates in lines 4 and 6 require no parameter communication.
The positive-negative phases are similar with the construction-reconstruction phases in \cite{Neftci2014}, where they propose an event-driven CD algorithm that uses SNN dynamics to train RBM.
They argue that one can use global signals to modulate the two learning phases on NMSs.
We fill in the gap between CGS and a NMS implementation by using a SNN dynamics to resample the topic, i.e. line 5.
\begin{algorithm}[t!]
    \caption{SpikeCGS, where $\tau_1(x) = \log( \exp x -1)$ and $\tau_2(x) = \log( \exp x +1)$.}
    \begin{algorithmic}[1]
        \Require{A corpus}
        \State Initialization: Using the initialization method in CGS. And then initialize the network parameters as:
        {\small
        \begin{align}
            M_{zw}^{\alpha} &= \log(C_{w,z}+\lambda_w),\forall w,z;\nonumber\\
            M_{zd}^{\beta} &= \log(C_{z,d}+\varphi_z),\forall d,z;~~~
            b_{z} = \log(C_{\cdot,z}+\bar{\lambda}),\forall z\nonumber
        \end{align}
        }%
        \Repeat
        \State Uniformly sample a token $w := w_{di}$ from the corpus, and 
                let $z'$ denote its last topic assignment.
        \State Negative phase:~neurons $x_w^{\alpha},x_{d}^{\beta}$ and $h_{z'}$ fire spikes and the connections between them are updated:
        {\small
        \begin{align}
             M_{z'w}^{\alpha} = \tau_1(M_{z'w}^{\alpha}),~~~
            M_{z'd}^{\beta} = \tau_1(M_{z'd}^{\beta}),~~~
            b_{z'} =& \tau_1(b_{z'})\nonumber
        \end{align}
        }%
        \State Resample:~$z$ = \Call{InferenceDynamics}{$w,d$}\label{alg.line.snn resample}
        \State Positive phase:
        {\small
        \begin{align}
             M_{zw}^{\alpha} = \tau_2(M_{zw}^{\alpha}),~~~
            M_{zd}^{\beta} = \tau_2(M_{zd}^{\beta}),~~~
            b_{z} =& \tau_2(b_{z})\nonumber
        \end{align}
        }%
        \Until{A pre-specified number of iterations is reached.}
    \end{algorithmic}
    \label{alg.SpikeCGS}
\end{algorithm}

\subsection{Online learning SNN: MAP}\label{sec.online snn}
SpikeCGS is a SNN implementation of CGS.
It should store the whole corpus and all topic assignments, which is undesirable on a storage-limited environment.
To improve the memory efficiency, we propose online algorithms.

In this section, we propose an online SNN algorithm that does stochastic optimization for problem (\ref{equ.MAP stochastic problem}).
As it is based on optimization instead of sampling, new tools are required to formalize the algorithm.
First a probabilistic model for the SNN variables is defined.
And then a conceived optimization problem on this probabilistic model is proposed and shown to be equivalent with problem (\ref{equ.MAP stochastic problem}).
Lastly, the optimizer to the conceived problem that is suitable for SNN implementation is proposed and analyzed.

\subsubsection{Define a probabilistic model}
A probabilistic model on the SNN is defined in Def.~\ref{def.distribution}.
It describes a joint distribution on the word layer $\mathbf{x}^{\alpha}$ and hidden topic layer $\mathbf{h}$ given the document layer $\mathbf{x}^{\beta}$.
To simplify notations, we define $\zeta(\mathbf{x})\!=\!\sum_{j=1}^J\exp(x_j)$, where $J$ is the dimensionality of $\mathbf{x}$.
\begin{definition}\label{def.distribution}
A distribution of the network variables is defined as:
{\small
\begin{align}
 p(x_w^{\alpha} &= 1, h_z = 1|x_d^{\beta} = 1;\mathbf{M})\nonumber \\
 = \exp\big[&\mathbf{M}_{z\cdot}^{\alpha}\cdot\mathbf{x}^{\alpha}  +  \mathbf{M}_{z\cdot}^{\beta}\cdot\mathbf{x}^{\beta}  - A(\mathbf{M}_{z\cdot}^{\alpha}, \mathbf{M}_{\cdot d}^{\beta})\big],\label{equ.SNN ML}
\end{align}
}%
where $A(\mathbf{M}_{z\cdot}^{\alpha}, \mathbf{M}_{\cdot d}^{\beta})= \log(\zeta(\mathbf{M}_{z\cdot}^{\alpha}))+\log(\zeta(\mathbf{M}_{\cdot d}^{\beta}))$ is the log-partition function to ensure normalization.
Note that $\mathbf{x}^{\alpha},\mathbf{x}^{\beta},\mathbf{h}$ are one-hot vectors.
\end{definition}

Before formalizing the learning problem, we show that Eq.~(\ref{equ.SNN ML}) is closely related to the LDA likelihood $p(w,z|d;\boldsymbol\Phi,\boldsymbol\Theta)$.

\begin{lemma}\label{prop.re-parametrization}(Proof in Appendix~\ref{proof.re-parameterization})
If the following conditions holds,
\begin{enumerate}
\item $\zeta(\mathbf{M}_{z\cdot}^{\alpha})\!=\!1,\forall z$ and $ \zeta(\mathbf{M}_{\cdot d}^{\beta})\!=\!\kappa
,\forall d$,
where $\kappa$ is some constant;
\item the variables and parameters are related by Tab.~\ref{tlb.Parameter relation}(a,c);
\end{enumerate}
then Eq.~(\ref{equ.SNN ML}) equals the complete likelihood that a word $w$ in document $d$ is assigned the topic $z$ in LDA:
{\small
\begin{align}
p(x_w^{\alpha} \!=\! 1, h_z \!=\! 1|x_d^{\beta} \!=\! 1;\mathbf{M}) = p(w,z|d;\boldsymbol\Phi,\boldsymbol\Theta),\label{equ.equivalent}
\end{align}
}%
Moreover, the conditional distribution of the topic assignment for a token is:
{\small
\begin{align}
    p(h_z=1|\mathbf{x}^{\alpha},\mathbf{x}^{\beta};\mathbf{M})&\propto\exp(u_z),\forall z,\label{equ.topic inference}
\end{align}
}%
where $u_z = \mathbf{M}_{z\cdot}^{\alpha} \cdot \mathbf{x}^{\alpha}  + \mathbf{M}_{z\cdot}^{\beta}\cdot\mathbf{x}^{\beta}$ is the weighted sum of the input $\mathbf{x}^{\alpha},\mathbf{x}^{\beta}$.
\end{lemma}

One consequence of Lemma~\ref{prop.re-parametrization} is that one can use Alg.~\ref{alg.E-step} to draw samples from the posterior Eq.~(\ref{equ.topic inference}).

\subsubsection{Conceived learning problem}
A constrained optimization problem is defined in Def.~\ref{def.SpikeLDA} to fit the model Eq.~(\ref{equ.SNN ML}) to streaming tokens $(w,d)$ from an unknown environment distribution $\pi(w,d)$.
And then we show this conceived problem is equivalent with the LDA problem (\ref{equ.MAP stochastic problem}) in Lemma~\ref{prop.MAP equivalence}.

\begin{definition}\label{def.SpikeLDA}
A SpikeLDA problem is defined as:
{\small
\begin{align}
    \max_{\mathbf{M}}&~~\mathbb{E}_{\pi(w,d)} \! \big[\log \! p(x_w^{\alpha}\!=\!1|x_d^{\beta}\!=\!1;\mathbf{M}) 
    \!+\! \frac{1}{N_d} \! \log p(\mathbf{M}^{\beta}_{\cdot d};\boldsymbol\lambda)\big],\nonumber\\
    \text{s.t. }&\zeta(\mathbf{M}_{z\cdot}^{\alpha}) = 1,\forall z,\quad \zeta(\mathbf{M}_{\cdot d}^{\beta}) = \kappa,\forall d,\label{equ.MAP SNN problem}
    \end{align}
}%
where $p(x_w^{\alpha}=1|x_d^{\beta}=1;\mathbf{M})$ is the marginal distribution with the latent topic layer variable $h_z$ summed out.
The prior is defined as:
{\small
\begin{align}
    p(\mathbf{M}^{\beta}_{\cdot d};\boldsymbol\lambda)&=\prod_{z=1}^K p(M_{zd}^{\beta}; \lambda_z),\forall d,\nonumber\\
    p(\exp(M_{zd}^{\beta});\lambda_z)&=\text{Gamma}((\lambda_z - 1),1),\forall d,z,
\end{align}
}%
and Gamma$(\cdot,\cdot)$ is the Gamma distribution.
\end{definition}

The objective function consists of the marginal likelihood and the prior defined for the network connections $\mathbf{M}^{\beta}$.
The problem is defined on some normalization manifold where using independent Gamma's is equivalent with using a Dirichlet~\cite{Aitchison1986}.
Now we show the conceived learning problem is equivalent with problem (\ref{equ.MAP stochastic problem}), the MAP problem of LDA.

\begin{lemma}\label{prop.MAP equivalence} (Proof in Appendix \ref{proof.MAP equivalence})
The SpikeLDA problem, Def.~\ref{def.SpikeLDA}, is equivalent to the LDA problem (\ref{equ.MAP stochastic problem}), when $\kappa\!=\!\sum_{z=1}^K (\varphi_z\!-\!1); \forall z, \varphi_z\!\geq\!0$ and using the parameter relations in Tab.~\ref{tlb.Parameter relation}(c).
\end{lemma}

In a related work, Nessler et al.~\shortcite{Nessler2013} conceive a constrained optimization problem on their SNN, which is equivalent with the online MLE of multinomial mixture (MM), and we denote it as the SpikeMM problem.
It is novel for SpikeLDA to extend from $D=1$ to $D>1$ and incorporate the prior by modifying the constraints and the objective simultaneously.
The idea of independent Gamma prior is critical for designing the optimizer suitable for NMS implementation, i.e., without intensive parameter communication.

\subsubsection{Optimization method}
A traditional way to train latent variable models in the online setting is the stochastic gradient EM~\cite{Cappe2009}.
But to compute the gradient for the topic models requires intensive parameter communication~\cite{Patterson2013}, which is inefficient for NMS implementation.
Alternatively, we devise new stochastic optimization methods based on "mean-limit" ordinary differential equation (ML-ODE).
Here we briefly introduce how this method works in general.

Suppose a stochastic parameter update rule
{\small
\begin{align}
    \mathbf{M}(t+1)\leftarrow\mathbf{M}(t) + \eta_t g(\mathbf{M}(t)),\label{equ.stochastic ODE}
\end{align}
}%
where $t$ denotes the discrete time, $g(\mathbf{M})$ is a noisy update direction, and $\eta_t$ is the step size.
The ML-ODE of this update rule is defined as
{\small
\begin{align}
    \text{ML-ODE:}~~~\frac{d}{ds}\mathbf{M}(s) = \mathbb{E}\big[g\big(\mathbf{M}(s)\big)\big],\label{equ.continuous ODE}
\end{align}
}%
where $s$ denotes the continuous time.
If all trajectories of the ML-ODE converge and the set of stable convergence points is the same as the set of local optima of an optimization problem, then we say that this ML-ODE solves the problem.
Moreover, if the stepsizes in Eq.~(\ref{equ.stochastic ODE}) satisfy the Robin-Monro condition $\sum_{t=1}^{\infty}\eta_t\!=\!\infty,\sum_{t=1}^{\infty}\eta_t^2\!<\!\infty$,
Kushner and Yin~\shortcite{Kushner2003} show that all sequences $\{\mathbf{M}(t)\}_{t=1}^{\infty}$ converge and the stochastic update rule Eq.~(\ref{equ.stochastic ODE}) can solve the optimization problem as well.
For instance, stochastic gradient EM~\cite{Cappe2009} is a special case where $g(\cdot)$ is chosen to be the gradient.

A stochastic parameter update rule is helpful when one can only deal with samples from some distributions to approximate the expectation in Eq.~(\ref{equ.continuous ODE}), such as when dealing with the unknown environment $\pi(w,d)$ and when using SNN dynamics Alg.~\ref{alg.E-step} to stochastically infer the topics.

The proposed online optimization algorithm is summarized in Alg.~\ref{alg.online algorithm}.
At each iteration (line~\ref{alg.repeate}), a single token is sampled (line~\ref{alg.alg3 token}) and a latent topic assignment is sampled (line~\ref{alg.alg3 resample}).
Once the sampling is finished, the corresponding synaptic weights are updated (line~\ref{alg.learning step}).
In the following, we prove that the algorithm solves the SpikeLDA problem~(\ref{equ.MAP SNN problem}).

\begin{algorithm}[t!]
    \caption{The online ed-SpikeLDA algorithm}
    \begin{algorithmic}[1]
        \Require{A corpus}
        \Require{The step sizes $\eta_t$ obey $\sum_t\eta_t=\infty,\sum_t\eta_t^2<\infty$}
        \State Randomly initialize $\mathbf{M}$.
        \Repeat\label{alg.repeate}
        \State Uniformly sample a token $w \triangleq w_{di}$ from the corpus\label{alg.alg3 token}
        \State Inference:~$z$ = \Call{InferenceDynamics}{$w,d$}\label{alg.alg3 resample}
        \State Learning:
        \begin{align}
             M_{zw}^{\alpha} &\leftarrow M_{zw}^{\alpha} + \eta_t h_z\left.(x_w^{\alpha}\exp(-M_{zw}^{\alpha})-1\right.),\forall w,z,\nonumber\\
             M_{zd}^{\beta} &\leftarrow M_{zd}^{\beta} + \eta_t x_d^{\beta}\big[( h_z+\frac{\lambda_z - 1}{N_d})\exp(-M_{zd}^{\beta})\nonumber\\&-\frac{1}{\kappa}-\frac{1}{N_d}\big],\forall d,z.\label{equ.MAP discrete dynamics}
        \end{align}
        \label{alg.learning step}
        \Until{A pre-specified number of iterations is reached.}
    \end{algorithmic}\label{alg.online algorithm}
\end{algorithm}

\begin{theorem}\label{prop.MAP stochastic update} (Proof in Appendix \ref{proof.ML dynamics})
In Alg.~\ref{alg.online algorithm}
the ML-ODE of the update rule Eq.~(\ref{equ.MAP discrete dynamics}) solve the SpikeLDA problem (\ref{equ.MAP SNN problem}).
So does the stochastic update rule Eq.~(\ref{equ.MAP discrete dynamics}). This algorithm is called the SpikeLDA algorithm.
\end{theorem}

We briefly explain how the theorem is proved and leave the details to the appendix.
A particle $\mathbf{M}$ starts at a randomly position in the parameter space.
The ML-ODE characterizes the particle's temporal dynamics in the space.
The particle undergoing our ML-ODE experience two successive phases.
In the first phase, $\mathbf{M}$ is driven monotonically\footnote{The monotonicity is defined in the proof. } to the manifold defined by the normalization constraints.
In the second phase, $\mathbf{M}$ travels on the manifold with $\frac{d}{ds}\mathbf{M}$ equals the natural gradient, until it reaches a local maximum.

\subsubsection{Remark 1: Locality}
It is obvious that the update of a synapse $M_{zd}^{\alpha}$ only locally depends on itself, its pre-synaptic neuron $x_w^{\alpha}$ and post-synaptic neuron $h_z$; and so does $M_{zd}^{\beta}$.
This locality behavior is called STDP in neuroscience and is suitable for NMS implementation.

In the related work, Nessler~\shortcite{Nessler2013} propose an optimizer to solve their SpikeMM problem.
The stochastic update rule for $\mathbf{M}^{\alpha}$ in our SpikeLDA is the same as theirs, but that for
$\mathbf{M}^{\beta}$ is a novel design.
It is inspired by the many independent structures of the Dirichlet distribution, particularly its relation with the Gamma distribution~\cite{Aitchison1986}.
The difficulty in designing the algorithm is the co-design of (1) a new problem equivalent with the original one (Lemma~\ref{prop.MAP equivalence}), and a update rule that (2) has a good local structure to implement on SNN and (3) the sequence $\{\mathbf{M}(t)\}_{t=1}^{\infty}$ converges to the manifold defined by the constrains and the local optimum simultaneously (Thm.~\ref{prop.MAP stochastic update}).

\subsubsection{Remark 2: Scalability}
In SpikeLDA, synaptic weights are updated once a latent neuron triggers a spike event.
This is called event-driven update, which has the advantage of reducing energy consumption~\cite{Merolla2014}.
However, the estimated stochastic update direction is noisy if only one latent sample is used, resulting in very slow convergence.
To develop a scalable algorithm, we replace event-driven update by delayed update.
But delayed update might introduce energy overhead to maintain the intermediate results.
We call the event-driven version as ed-SpikeLDA, and the delayed update one as du-SpikeLDA.

Specifically, du-SpikeLDA processes a mini-batch of tokens before the parameters update once.
First, it stochastically infers the topic assignments like ed-SpikeLDA, but maintains the samples until the whole mini-batch is processed.
Then the parameters are updated using the average of results from this mini-batch, resulting in
an unbiased estimate of Eq.~(\ref{equ.MAP discrete dynamics}) of less variance.

\subsubsection{Remark 3: pLSI}
pLSI is a special case of LDA.
We propose an ed-SpikePLSI algorithm as a direct extension of ed-SpikeLDA to train pLSI in the Appendix.

\subsection{Online learning SNN : ML on semi-collapsed distribution}
The SpikeLDA does point estimate in the joint space of $\{\boldsymbol\Phi,\boldsymbol\Theta\}$.
When delayed update is used, we can propose another learning algorithm where the local parameter $\boldsymbol\Theta$ is integrated out, to provide a better parameter estimation.

\subsubsection{Conceived learning problem}
This new algorithm optimizes the semi-collapsed likelihood Eq.~(\ref{equ.semi stochastic problem}).
We re-parameterize it to a constrained optimization problem on the SNN, which resembles how SpikeLDA is developed.
\begin{lemma}\label{prop.semi equivalence} (Proof in Appendix)
A semi-SpikeLDA problem is defined as:
{\small
 \begin{align}
    \max_{\mathbf{M}^{\alpha}}\mathbb{E}_{\pi(d)}&\mathbb{E}_{p(\mathbf{z}_d|\mathbf{w}_d;\mathbf{M}^{\alpha},\mathbf{M}^{\beta})}[C_{d,z,w}]\log p(x_w^{\alpha}=1|h_z=1;\mathbf{M}^{\alpha}),\nonumber\\
    &\text{s.t.}~~\zeta(\mathbf{M}_{z\cdot}^{\alpha})=1,~~~\forall z.\label{equ.semi SNN problem}
    \end{align}
}%
The semi-SpikeLDA problem is equivalent to the LDA problem (\ref{equ.semi stochastic problem}), when using the parameter relations in Tab.~\ref{tlb.Parameter relation}(d).
\end{lemma}

At each iteration, this SNN algorithm subsamples a mini-batch $\hat{D}$ of documents, performs semi-CGS on the topic assignments, and then optimizes the global parameters.
The implementation of semi-CGS mimics SpikeCGS and the implementation of optimization mimics SpikeLDA.
So the new algorithm is a hybrid of these two.
We outline how the semi-CGS and the optimization are implemented on SNN below, and leave the complete algorithm Alg.~\ref{alg.semiSpikeLDA} in the Appendix.
\subsubsection{semi-CGS}
At each iteration, suppose a token $w_{di}$ is sampled and its topic assignment should be resampled from Eq.~(\ref{equ.semi cgs update rule}).
If we represent the word $w_{di}$, its document $d$ and its topic assignment $z$ using the network variables via Tab.~\ref{tlb.Parameter relation}(a), the conditional distribution to sample from is
{\small
\begin{align*}
p(h_z\!=\!1|\cdot)
           \propto \phi_{zw} \cdot (C_{z,d}^{\neg di} + \lambda_z),~~~\forall z
\end{align*}
}%
on the SNN.
Furthermore, if the network parameters relate to the parameters and counts via Tab.~\ref{tlb.Parameter relation}(d), the conditional distribution becomes
{\small
\begin{align}
    p(h_z=1|\cdot) \propto \exp(M_{zw_{di}}^{\alpha} + M_{zd}^{\beta}),~~~\forall z.\label{equ.cgs snn inference generalization}
\end{align}
}%
Then the SNN implementation of semi-CGS is similar with SpikeCGS by using Alg.~\ref{alg.E-step} to sample the topic assignment, and the negative and positive phases to eliminate and update the counts represented by $\mathbf{M}^{\beta}$.
During this procedure, we collect some statistics: $\hat{N}_{z,w}$ denotes the number of times a word $w$ is assigned topic $z$ within $T$ iterations on the mini-batch $\hat{D}$, and $\hat{N}_{z}\triangleq\sum_{w=1}^V\hat{N}_{z,w}$.

\subsubsection{Optimization}
After collecting the statistics, $\mathbf{M}^{\alpha}$ is updated as
{\small
 \begin{align}
M_{zw}^{\alpha}\!\leftarrow M_{zw}^{\alpha}\!+\!\eta_t \frac{1}{|\hat{D}|T}\Big[\hat{N}_{z,w}\exp(-M_{zw}^{\alpha})\!-\!\hat{N}_{z}\Big],\forall w,z,\label{equ.semi-likelihood discrete dynamics}
    \end{align}
}%
It is the same as Eq.~(\ref{equ.MAP discrete dynamics}) of SpikeLDA, except that the former uses an empirical average over samples (delayed update) and the latter uses only one sample (event-driven).

The following theorem shows that the semi-SpikeLDA algorithm solves the semi-SpikeLDA problem (\ref{equ.semi SNN problem}).
The main idea of the proof is the same as Thm.~\ref{prop.MAP stochastic update}.

\begin{theorem}\label{prop.semi stochastic update} (Proof in Appendix \ref{proof.ML dynamics})
In Alg.~\ref{alg.semiSpikeLDA}, the ML-ODE of the update rule Eq.~(\ref{equ.semi-likelihood discrete dynamics}) solves the semi-SpikeLDA problem (\ref{equ.semi SNN problem}).
So does the stochastic update rule Eq.~(\ref{equ.semi-likelihood discrete dynamics}).
This algorithm is called the semi-SpikeLDA algorithm.
\end{theorem}

\section{Experiments}\label{sec.Experiment}

\begin{figure}[!t]
    \centering
    \captionsetup[subfigure]{labelformat=empty}
    \begin{subfigure}{0.15\textwidth}
        \centering
        \caption{KOS}
        \includegraphics[width=1.2\textwidth]{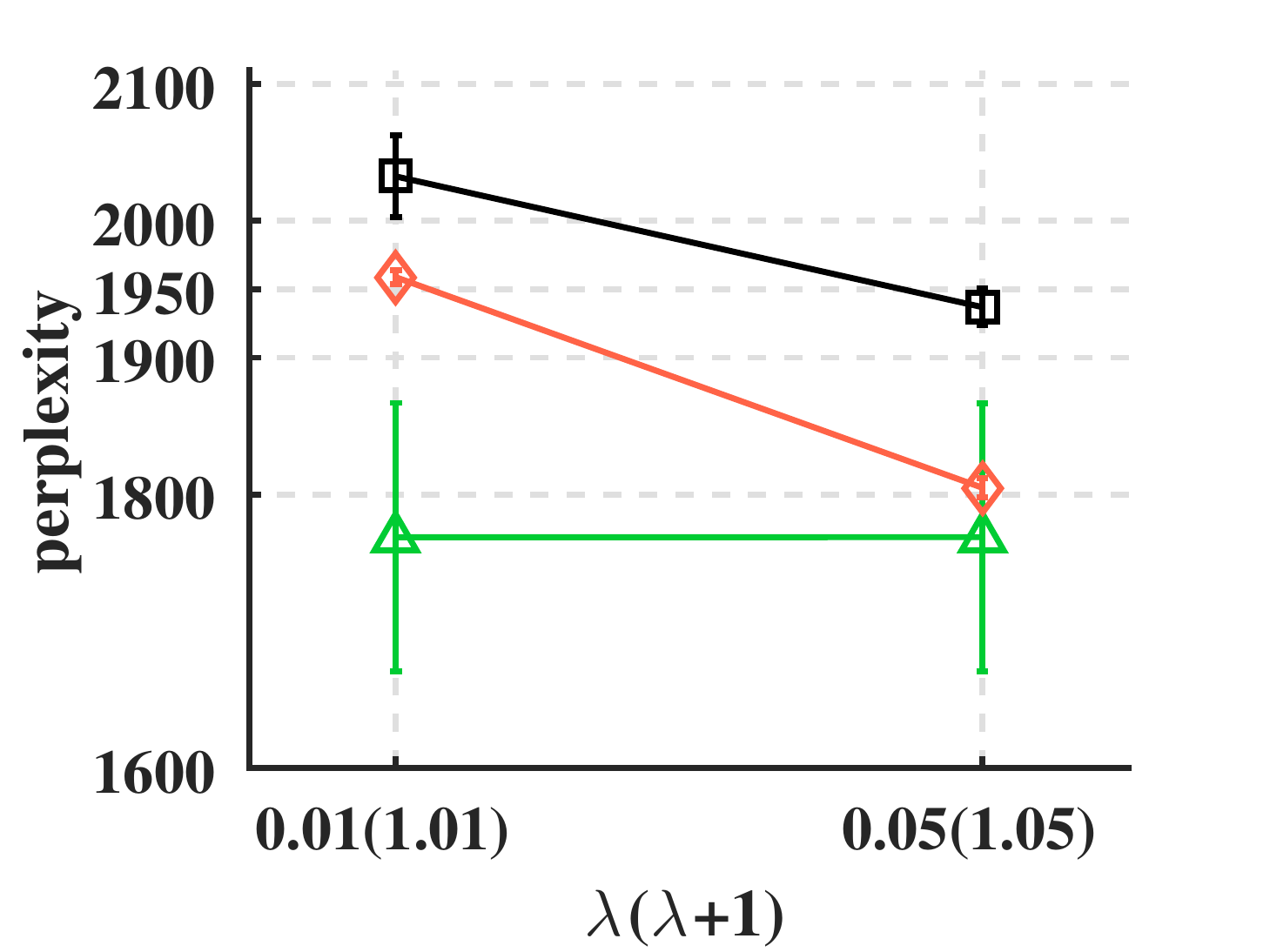}
    \end{subfigure}
    \begin{subfigure}{0.15\textwidth}
        \centering
        \caption{20NG}
        \includegraphics[width=1.2\textwidth]{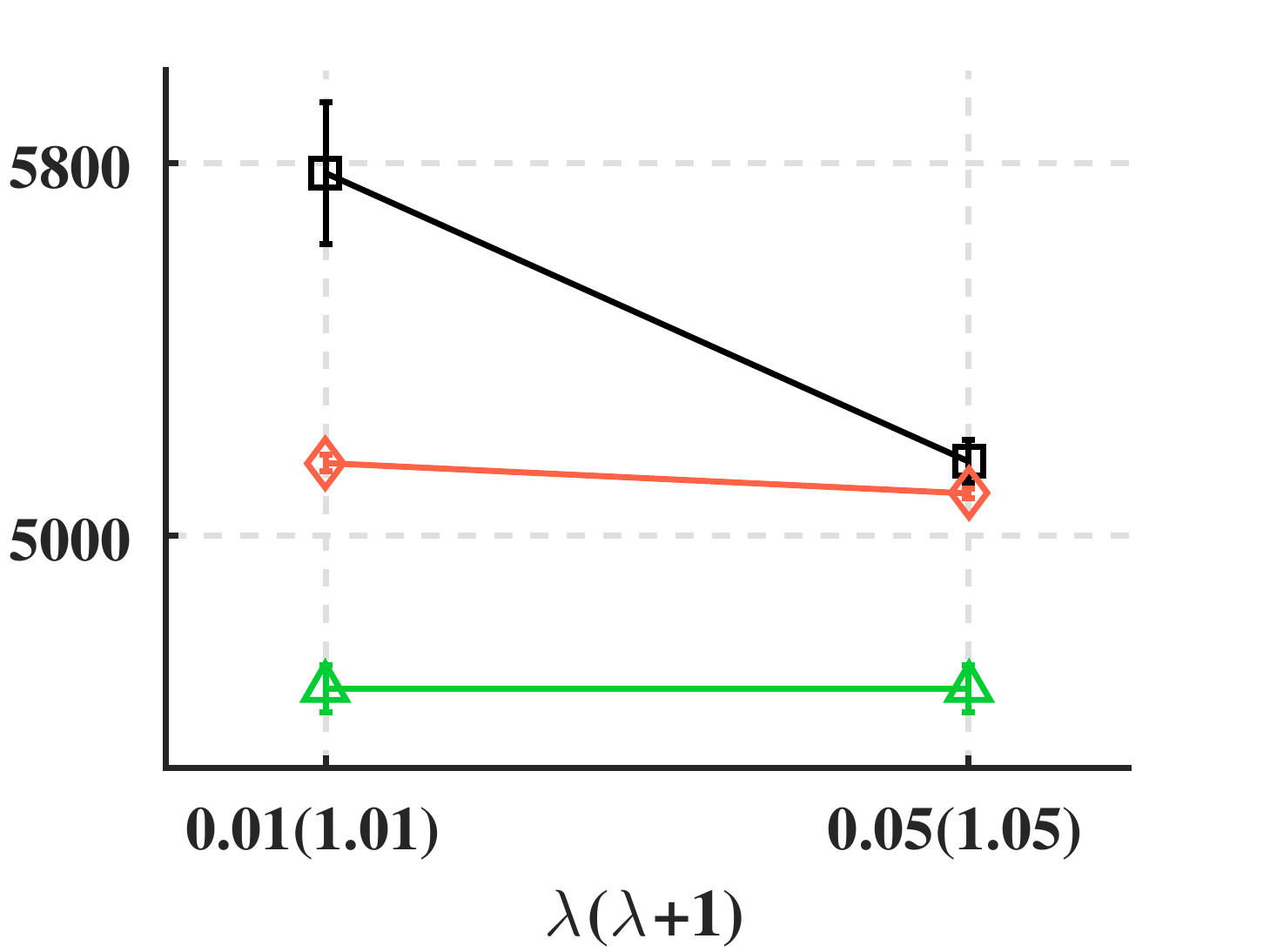}
    \end{subfigure}
     \begin{subfigure}{0.15\textwidth}
        \centering
        \caption{Enron}
        \includegraphics[width=1.2\textwidth]{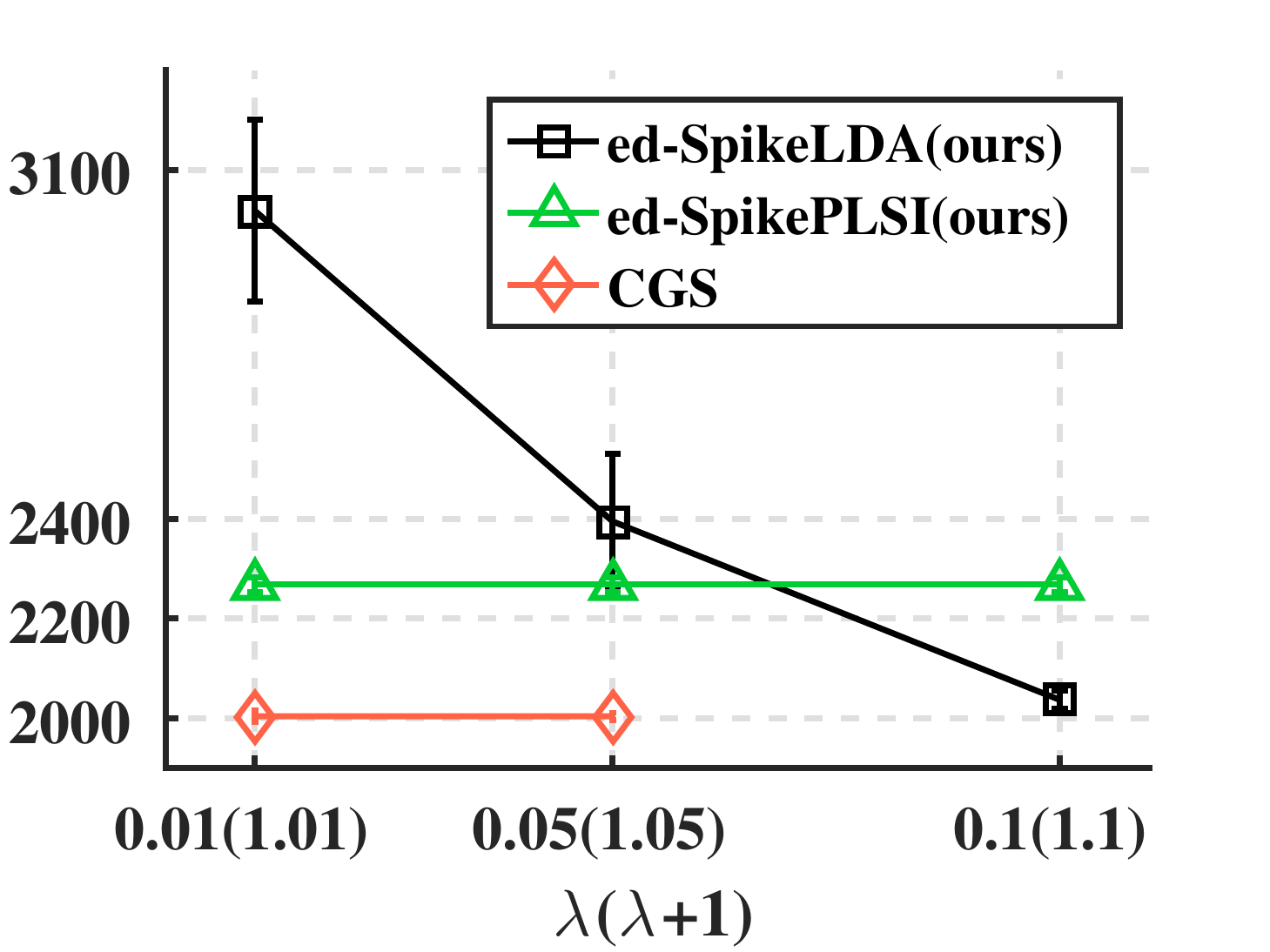}
    \end{subfigure}
    \caption{Impact of $\lambda$. Small datasets. In the x-axis, $\lambda(1+\lambda)$ means  $\lambda$ is for CGS and $\lambda+1$ is for ed-SpikeLDA; ed-SpikePLSI has no $\lambda$ to tune. $K=200$.}
    \label{fig.Final Perplexity}
\end{figure}

\begin{figure}[!t]
    \centering
    \captionsetup[subfigure]{labelformat=empty}
    \begin{subfigure}{0.23\textwidth}
        \centering
        \caption{KOS}
        \includegraphics[width=\textwidth]{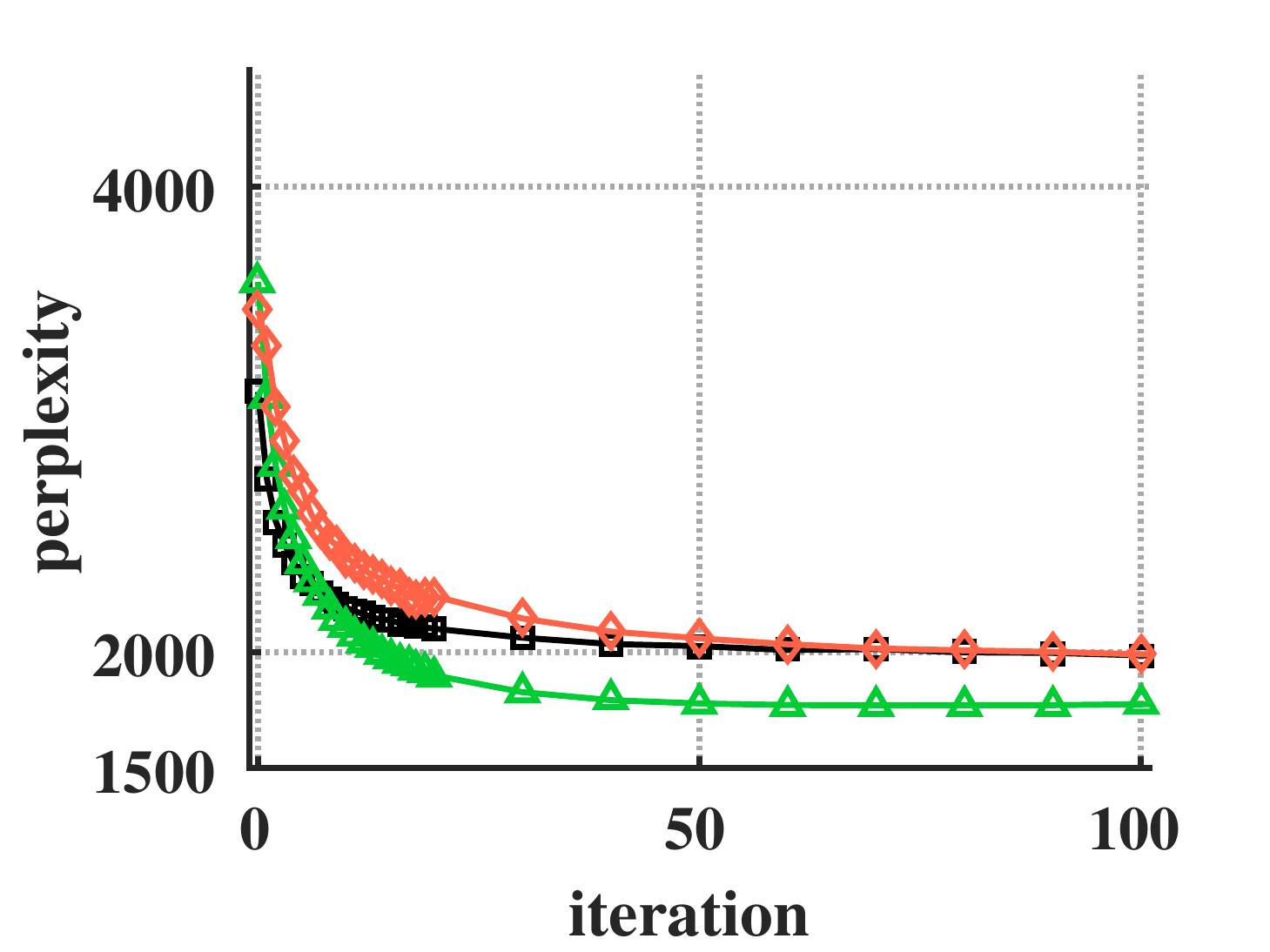}
    \end{subfigure}
    \hspace{0mm}
    \begin{subfigure}{0.23\textwidth}
        \centering
        \caption{20NG}
        \includegraphics[width=\textwidth]{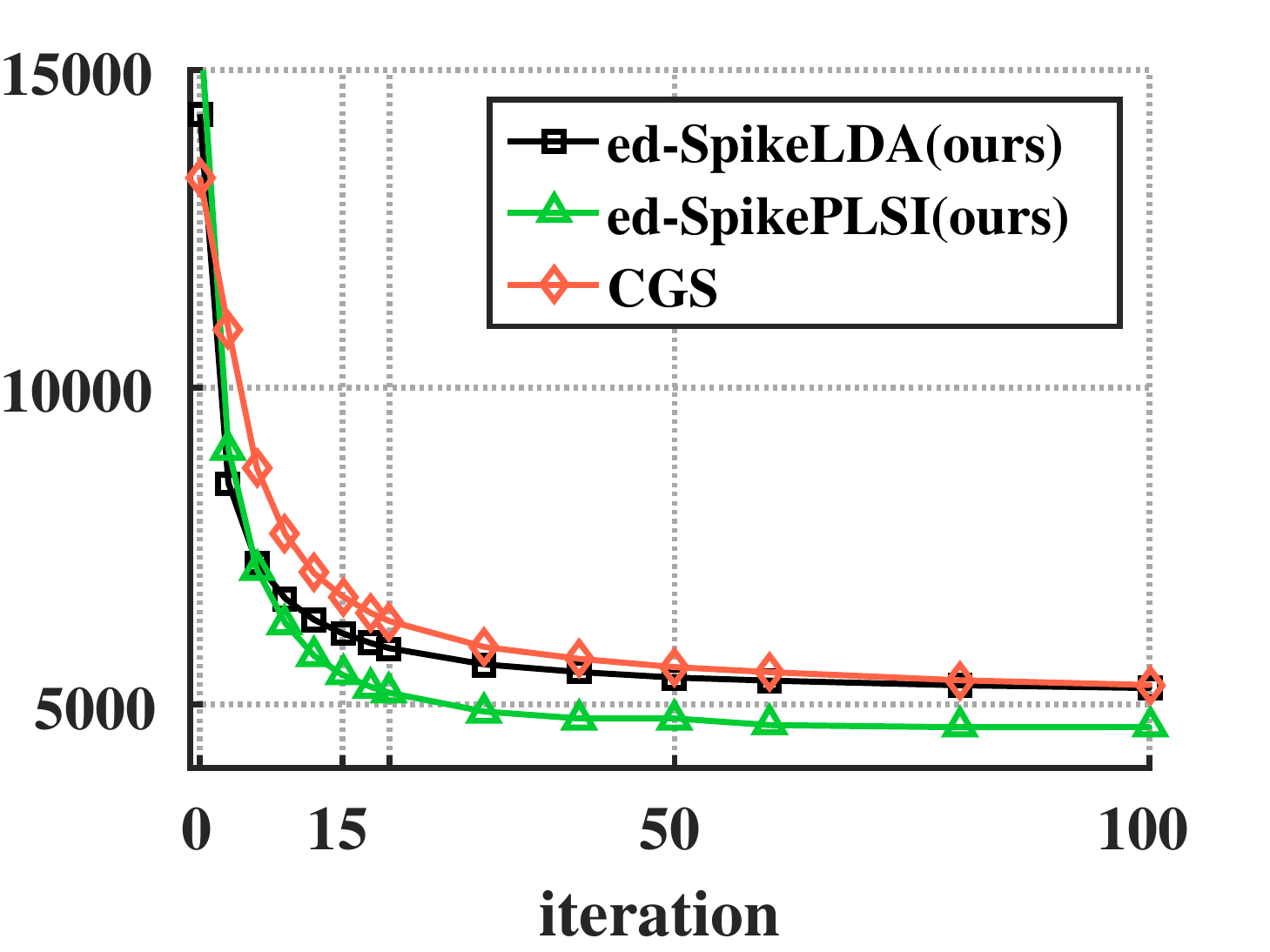}
    \end{subfigure}
    \caption{Convergence speed. Small datasets. $K=200$. Error bars are small and omitted.}
    \label{fig.Perplexity vs iteration}
\end{figure}

In the experiment, we assess (1) the generalization performance and (2) the discriminative power of the learned document representations of the proposed online SNN algorithms.
Because SpikeCGS is an implementation of CGS, we don't particularly do experiments on it.
All algorithms are simulated on GPCs.
Evaluating and optimizing their runtime performance on NMSs will be a future work.

The datasets are KOS, Enron, NIPS, 20NG, NyTimes, and Pubmed.
The NIPS results and statistics of the datasets are summarized in the appendix.

The baselines are GPC algorithms to train smoothed LDA, including CGS~\cite{Griffiths2004}, stochastic VI~\cite{Hoffman2013} and SGMCMC~\cite{Patterson2013}.
We use symmetric Dirichlet prior, $\boldsymbol\lambda = \lambda\mathbf{1}$ and $\boldsymbol\varphi = \varphi\mathbf{1}$.
$\lambda$ and $\varphi$ are offset by $+0.5$ for VI and $+1$ for MAP, including our SpikeLDA.
To evaluate the generalization performance, we use fold-in method to calculating perplexity following~\cite{Asuncion2009}.
Specifically, for the baselines, VI-based methods use the alternative estimate which generally reports lower perplexity; CGS and SGMCMC use only one sample as estimate.
For our new SNN models, the weights $\mathbf{M}^{\alpha}, \mathbf{M}^{\beta}$ are first converted to $\boldsymbol\Phi,\boldsymbol\Theta$ following Tab.~\ref{tlb.Parameter relation}, which is then used to
calculate the perplexity following~\cite{Asuncion2009} or train classifiers.

All results are averaged from $3$ different runs of the algorithms.
More details of the experiments are in the Appendix.

\subsection{Perplexity results}\label{sec.convergence behavior}
On small datasets where the ed-SpikeLDA can handle in an acceptable time, we compare convergence speeds and the final perplexities of ed-SpikePLSI and ed-SpikeLDA.
Fig.~\ref{fig.Final Perplexity} shows the final perplexities with varying hyper-parameter $\lambda$.
The ed-SpikePLSI usually finds good solutions.
But ed-SpikeLDA can outperform it by tuning its hyper-parameter $\lambda$ as observed in the Enron experiment.
Moreover, Fig.~\ref{fig.Perplexity vs iteration} compares the convergence speeds under the hyper-parameter setting when they converge to similar results.

\begin{figure}[t!]
        \begin{subfigure}{0.23\textwidth}
            \centering
            \caption{NYTimes}
            \includegraphics[width=\textwidth]{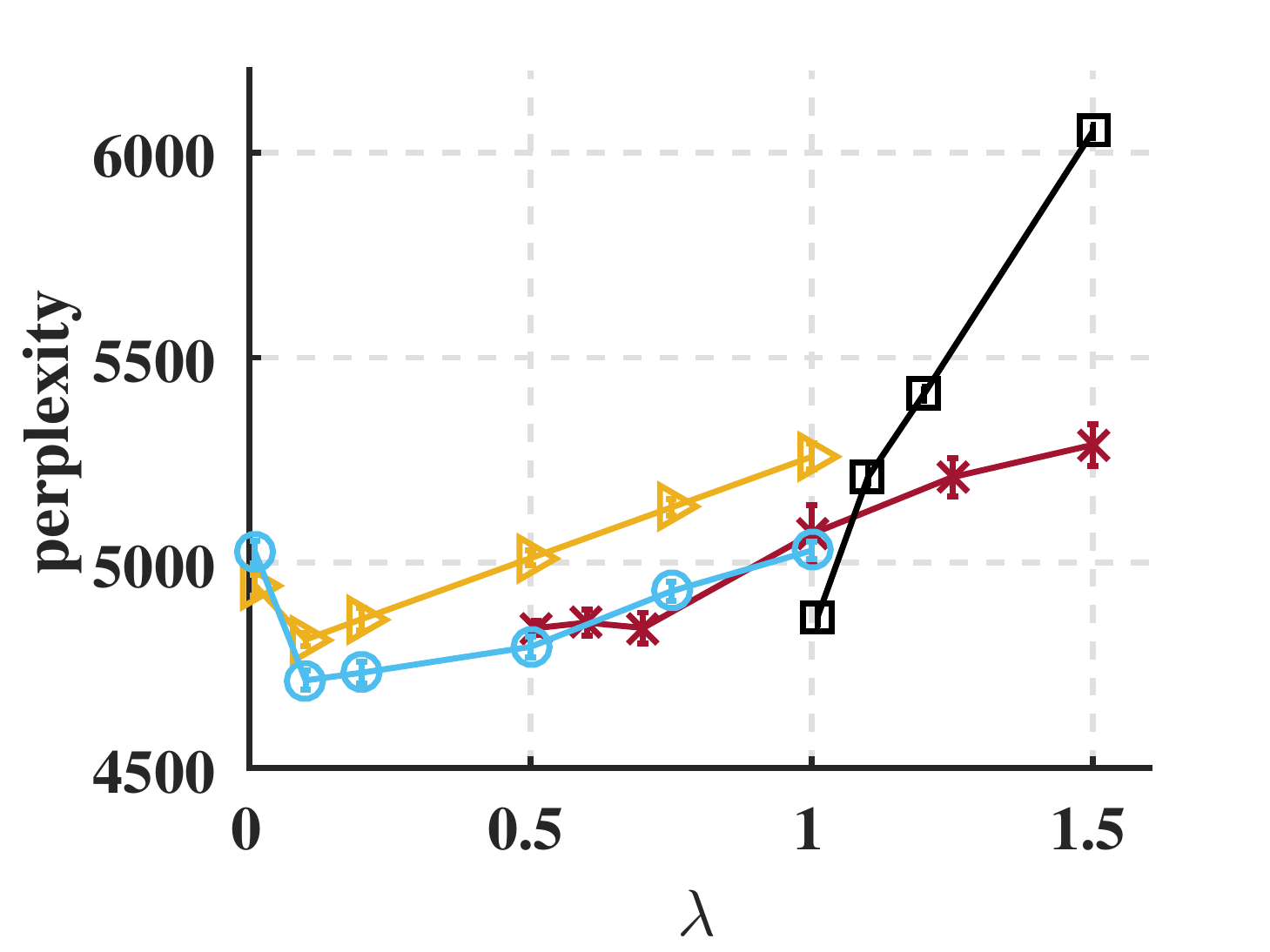}
        \end{subfigure}
        \begin{subfigure}{0.23\textwidth}
            \centering
            \caption{Pubmed}
            \includegraphics[width=\textwidth]{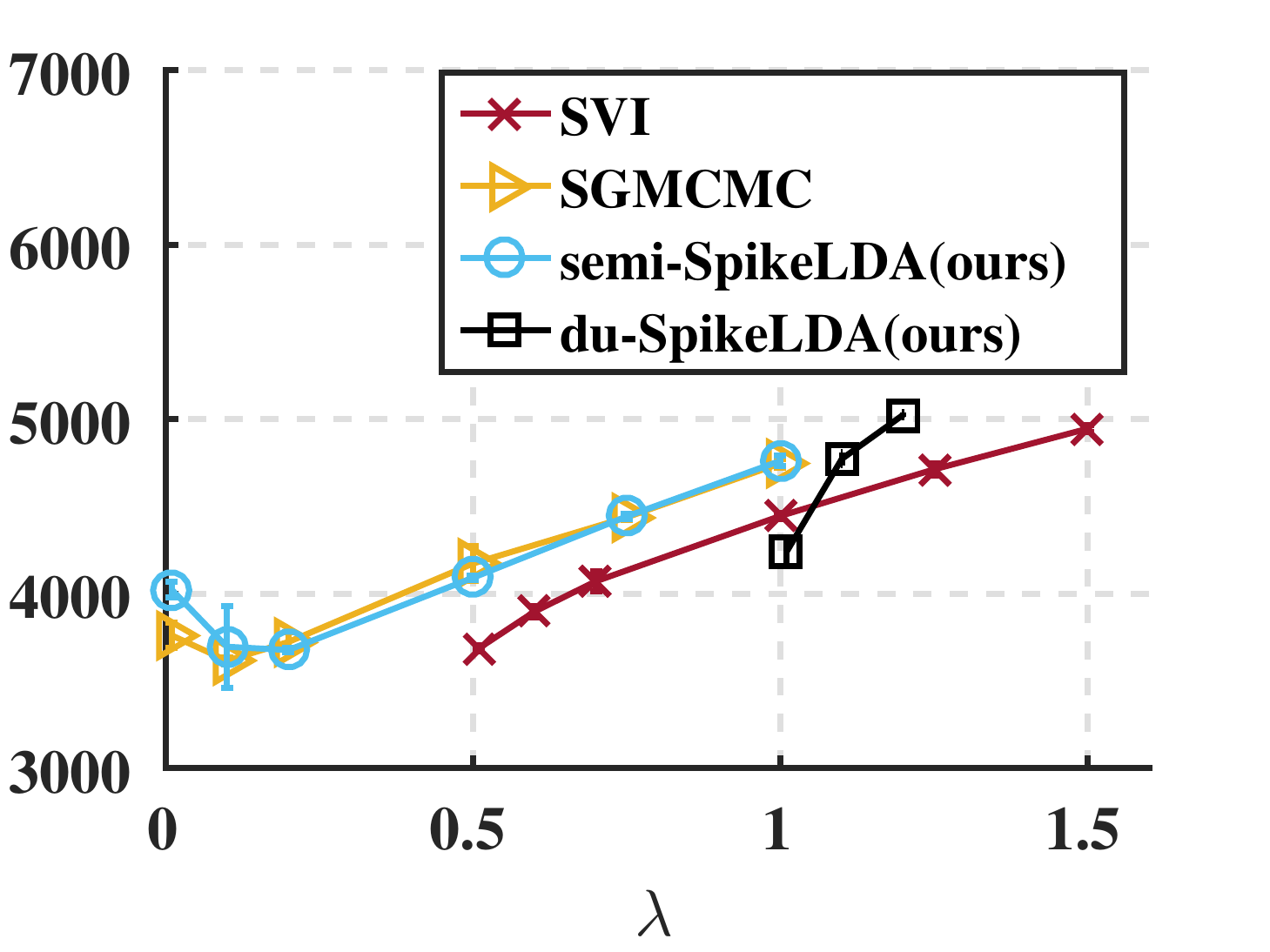}
        \end{subfigure}
        \caption{Impact of $\lambda$. Large datasets. Iterations number: $3000$ for NYTimes, $5000$ for Pubmed. $K=50$.}
        \label{fig.large final solution}
\end{figure}
\begin{figure}[t!]
        \begin{subfigure}{0.23\textwidth}
            \centering
            \caption{NYTimes}
            \includegraphics[width=\textwidth]{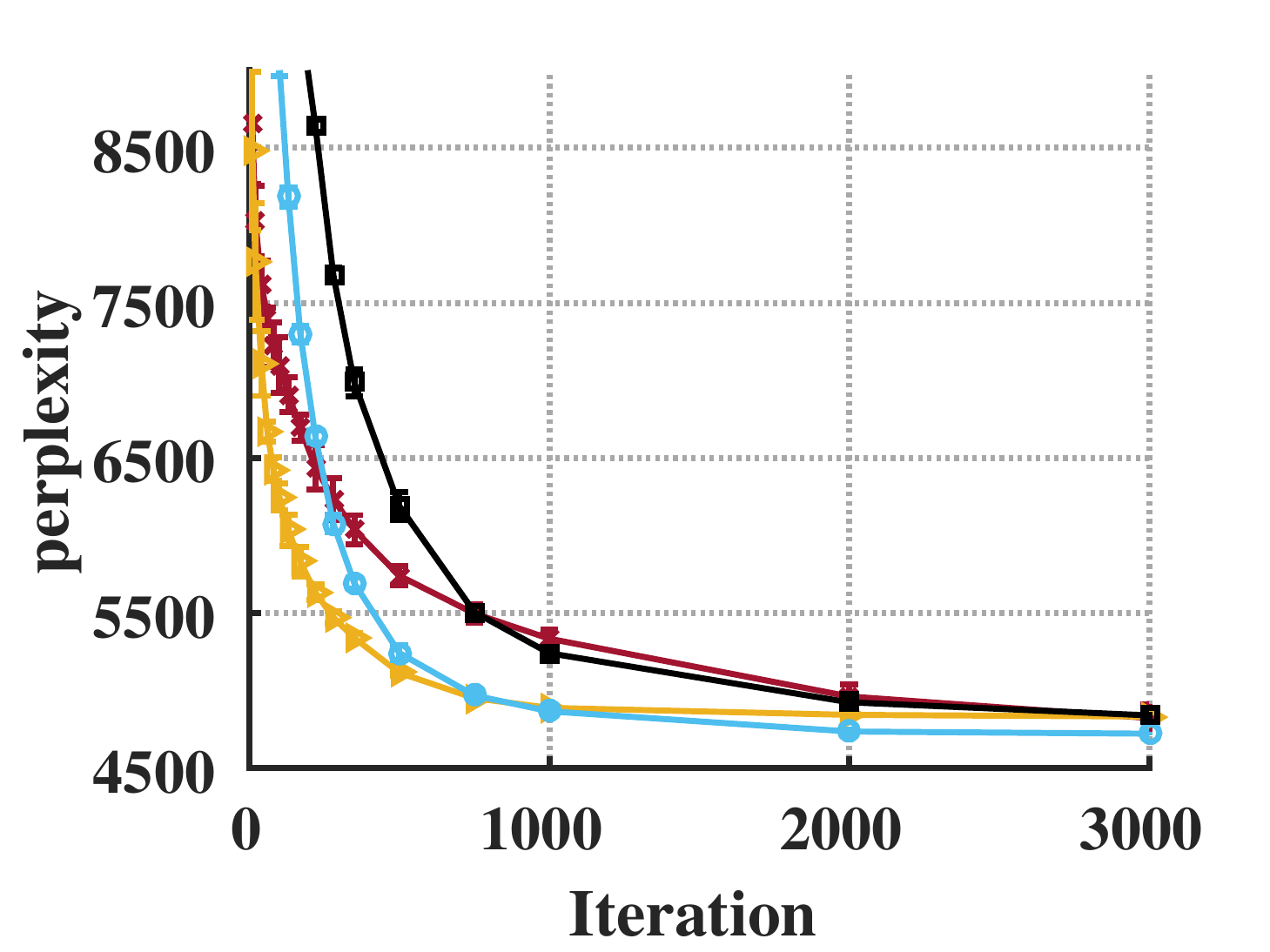}
        \end{subfigure}
        \begin{subfigure}{0.23\textwidth}
            \centering
            \caption{Pubmed}
            \includegraphics[width=\textwidth]{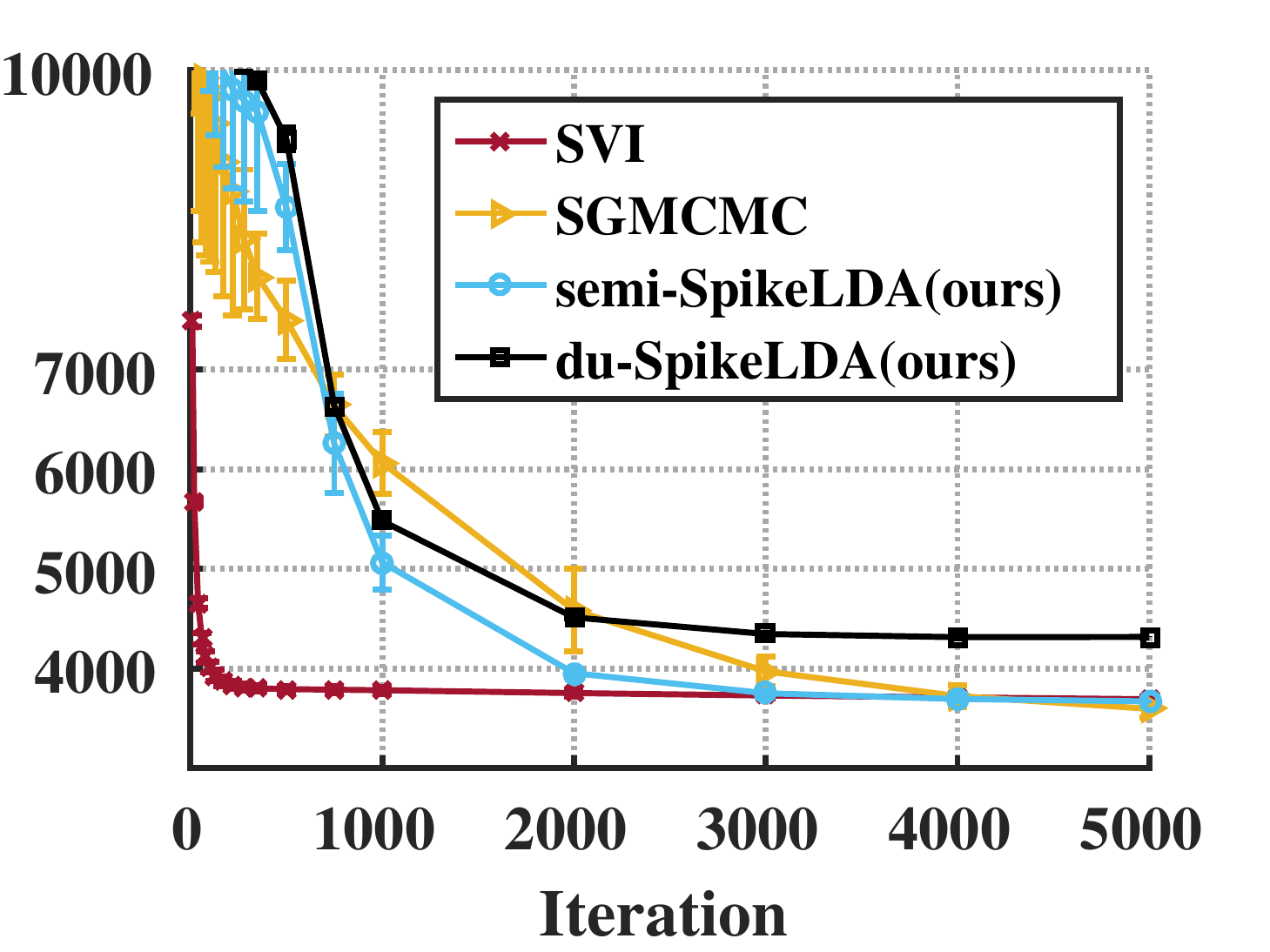}
        \end{subfigure}
        \caption{Convergence results. Large datasets. $K=50$.}
        \label{fig.large convergence}
\end{figure}

On larger datasets, we compare the delayed update algorithms, du-SpikeLDA and semi-SpikeLDA, with some existing stochastic algorithms.
Fig.~\ref{fig.large final solution} shows that the SNN algorithms find similar solutions with the baseline methods after a fixed number of iterations of training.
The semi-SpikeLDA is less sensitive to hyper-parameters than du-SpikeLDA and competitive with SVI and SGMCMC.
Fig.~\ref{fig.large convergence} moreover shows the convergence behaviors of different algorithms under the hyper-parameter settings when they converge to similar results.
Our SNN algorithms perform reasonably while they are suitable for NMS implementations.

\subsection{Discriminative results}
We examine the discriminative ability of the learnt latent representations on 20NG (see Tab.~\ref{tlb.classifcation}).
The latent representations of the training documents are used as features to build a binary/multi-class SVM classifier.
As in \cite{Zhu2009}, the binary classification is to distinguish groups \textit{alt.athesism} and \textit{talk.religion.misc}.
We use the LIBLINEAR tool-kit \cite{Fan2008} and choose the L2-regularized L1-loss with $C=1$ to build the SVM.
We set $\lambda = 0.05$ for CGS and $\lambda=1.05$ for ed-SpikeLDA.
The network outputs useful representations for discrimination.
\begin{table}[h!]
    \centering
    \resizebox{0.33\textwidth}{!}{%
    \begin{tabular}{|c|c c |}\hline
       & Binary &  Mutli-class\\ \hline
    CGS& $72.8\pm 3.8$& $63.5\pm 0.2$ \\
    ed-SpikeLDA&  $72.5\pm 2.4$& $57.6\pm 3.3$\\\hline
    \end{tabular}%
    }
    \caption{Classification accuracy on 20NG.}
    \label{tlb.classifcation}
\end{table}

\section{Conclusions and Future work}\label{Sec.Conclusion}
We propose three SNN algorithms to train LDA, which are competitive to the GPC algorithms in generalization performance and discriminative power.
The first one is a batch algorithm based on CGS.
The second one is an online algorithm based on optimization.
The last one is a hybrid of the former two algorithms, but uses delayed update.
A future work is to assess the algorithms on real NMSs.
As the first step, we propose a network pruning method for ed-SpikeLDA in the Appendix.

\section{Acknowledgements}
This work is supported by the National NSF of China (Nos. 61620106010, 61621136008, 61332007). We thank Nan Jiang for helpful discussions.

{\small
\bibliographystyle{aaai}
\bibliography{Xiao}
}

\clearpage
\onecolumn

\appendix

\section{NMS and event-driven update}
\subsection{An introduction to NMS}
Artificial intelligence (AI) algorithms have become the most important computing workload nowadays.
Special hardware is developed to improve the efficiency when implementing some restricted but important types of AI algorithms.
Among the new hardware, we focus on an emerging hardware technology called the neuromorphic multi-chip system (NMS).
NMSs are developed for low-energy implementations of the spiking neural network (SNN) algorithms, which are originally developed to mathematically model the computations in the biological brain.
NMSs have some special designs making it nontrivial to implement an ordinary algorithm.
In the following, we summarize their hardware characteristics we concern about.

A NMS is a multi-chip system.
There are computational units and memory units on each chip.
The computational units implement the neurons, and a major kind of memory units implement the synapses (to store and fast retrieve the weights).
Inter-chip communication is effective for spikes.
\begin{enumerate}
\item Computation: A hardware neuron calculates a weighted sum of the input and output spikes following some specific activation methods asynchronously~\cite{Merolla2014}\footnote{The activation methods in SNN are also commonly called SNN dynamics.}.
\item Communication: Memory units are distributive and co-located with computational units over the chips.
The memory on a chip is typically limited, e.g., 52MB and 128MB~\cite{Merolla2014,Furber2014}.
Therefore if a model is large, its neurons and model parameters reside on several chips and requires inter-chip communication, which is only efficient for binary spikes by using the AER protocol~\cite{Mahowald1994}.
Communicating high-precision numbers is either impossible or inefficient.
\item Storage:
Even though a NMS can be integrated with an external memory~\cite{Connor2013} to store some infrequently access data, the external memory is usually limited on devices where low-energy computing is required.
\end{enumerate}

\subsection{Event-driven update}
SpikeCGS and ed-SpikeLDA are event-driven algorithms, that the synaptic weights are updated once a latent neuron fires a spike.
This resembles the STDP synaptic update rule in neuroscience and is suitable for NMS implementation.
The synapses that are modified when a latent neuron fires a spike is shown in Fig.~\ref{fig.synapse update}.
\begin{figure}[!h]
    \vspace{-.4cm}
    \centering
    \captionsetup[subfigure]{labelformat=empty}
    \begin{subfigure}{0.23\textwidth}
        \centering
    \vspace{0pt}
    \resizebox{\textwidth}{!}{
    \begin{tikzpicture}[node distance=\nodedist]
        \def\nodesize{15pt};

        \tikzstyle{every pin edge}=[<-,shorten <=2pt]
        \tikzstyle{neuron}=[circle, draw=black!100, line width=0.5mm, inner sep=0pt, minimum size=\nodesize]
        \tikzstyle{dark_neuron}=[circle, fill=black!100, draw=black!100, line width=0.5mm, inner sep=0.5pt, minimum size=\nodesize]
        \tikzstyle{grey_neuron}=[circle, draw=black!100, line width=1.75mm, inner sep=0.5mm, minimum size=\nodesize - 1mm]
        \tikzstyle{dot}=[circle, draw=black!100, fill=black!100, line width=0.5pt, inner sep=0.5pt, minimum size=0.1mm]
        \tikzstyle{annot} = [text width=10em, text centered]

        \foreach \name / \x in {0,...,5}
            \path[xshift=-2cm]
                node[neuron] (w-\name) at ({-\x * \nodesize * 1}, 0) {};
        \foreach \name / \x in {1}
            \path[xshift=-2cm]
                node[grey_neuron] (w-\name) at ({-\x * \nodesize * 1}, 0) {};
        \draw[xshift=-2cm, line width = 0.5mm] ({0.5 * \nodesize}, {0.5 * \nodesize + 0.1}) rectangle ({-5.5 * \nodesize}, {-0.5 * \nodesize - 0.1}) node[right = {3 * \nodesize}, below = {0.1*\nodesize}]{$\mathbf{x}^{\alpha}$};

        \foreach \name / \x in {0,...,3}
            \path[xshift=-0.5cm]
                node[neuron] (d-\name) at ({\x * \nodesize * 1}, 0) {};
        \foreach \name / \x in {1}
            \path[xshift=-0.5cm]
                node[grey_neuron] (d-\name) at ({\x * \nodesize * 1}, 0) {};
        \draw[xshift=-0.5cm, line width = 0.5mm] ({-0.5 * \nodesize}, {0.5 * \nodesize + 0.1}) rectangle ({3.5 * \nodesize}, {-0.5 * \nodesize - 0.1}) node[left = {2*\nodesize}, below = {0.1*\nodesize}]{$\mathbf{x}^{\beta}$};

        \foreach \name / \x in {0,2}
            \path[xshift=-2.25cm]
                node[neuron] (t-\name) at ({\x * \nodesize * 1}, {3 * \nodesize}) {};
         \foreach \name / \x in {1}
            \path[xshift=-2.25cm]
                node[dark_neuron] (t-\name) at ({\x * \nodesize * 1}, {3 * \nodesize}) {};
        \draw[xshift=-2.25cm, yshift={3 * \nodesize}, line width = 0.5mm] ({-0.5 * \nodesize}, {0.5 * \nodesize + 0.1}) rectangle ({2.5 * \nodesize}, {-0.5 * \nodesize - 0.1}) node[left = {1.5 * \nodesize}, above = \nodesize] {$\mathbf{h}$}; 

        \foreach \name / \x in {0,...,5}
            \path[xshift=-2cm]
                node[neuron] (w-\name) at ({-\x * \nodesize * 1}, 0) {};
        \draw[xshift=-2cm, line width = 0.5mm] ({0.5 * \nodesize}, {0.5 * \nodesize + 0.1}) rectangle ({-5.5 * \nodesize}, {-0.5 * \nodesize - 0.1}) node[right = {3 * \nodesize}, below = {0.1*\nodesize}]{$\mathbf{x}^{\alpha}$};

         \foreach \name / \x in {1}
                \draw[line width =0.7mm,>=latex,->,draw=black!100] ({-2cm - \x * \nodesize}, {\nodesize * 0.5}) -- ({-2.1cm + \nodesize * 0.7}, {2.5 * \nodesize});


        \foreach \name / \x in {1}
            \draw[line width =0.7mm,>=latex,->,draw=black!100] ({-0.5cm + 1 * \nodesize}, {\nodesize * 0.5}) -- ({-2.25cm + \x * \nodesize * 1}, {2.5 * \nodesize});

    \end{tikzpicture}
    }
    \vspace{-.6cm}
    \caption{(a)~SpikeCGS}
    \end{subfigure}
    \hspace{0mm}
    \begin{subfigure}{0.23\textwidth}
    \centering
    \resizebox{1\textwidth}{!}{
    \begin{tikzpicture}[node distance=\nodedist]
        \def\nodesize{15pt};

        \tikzstyle{every pin edge}=[<-,shorten <=2pt]
        \tikzstyle{neuron}=[circle, draw=black!100, line width=0.5mm, inner sep=0pt, minimum size=\nodesize]
        \tikzstyle{dark_neuron}=[circle, fill=black!100, draw=black!100, line width=0.5mm, inner sep=0.5pt, minimum size=\nodesize]
        \tikzstyle{grey_neuron}=[circle, draw=black!100, line width=1.75mm, inner sep=0.5mm, minimum size=\nodesize - 1mm]
        \tikzstyle{dot}=[circle, draw=black!100, fill=black!100, line width=0.5pt, inner sep=0.5pt, minimum size=0.1mm]
        \tikzstyle{annot} = [text width=10em, text centered]

        \foreach \name / \x in {0,...,5}
            \path[xshift=-2cm]
                node[neuron] (w-\name) at ({-\x * \nodesize * 1}, 0) {};
        \foreach \name / \x in {1}
            \path[xshift=-2cm]
                node[grey_neuron] (w-\name) at ({-\x * \nodesize * 1}, 0) {};
        \draw[xshift=-2cm, line width = 0.5mm] ({0.5 * \nodesize}, {0.5 * \nodesize + 0.1}) rectangle ({-5.5 * \nodesize}, {-0.5 * \nodesize - 0.1}) node[right = {3 * \nodesize}, below = {0.1*\nodesize}]{$\mathbf{x}^{\alpha}$};

        \foreach \name / \x in {0,...,3}
            \path[xshift=-0.5cm]
                node[neuron] (d-\name) at ({\x * \nodesize * 1}, 0) {};
        \foreach \name / \x in {1}
            \path[xshift=-0.5cm]
                node[grey_neuron] (d-\name) at ({\x * \nodesize * 1}, 0) {};
        \draw[xshift=-0.5cm, line width = 0.5mm] ({-0.5 * \nodesize}, {0.5 * \nodesize + 0.1}) rectangle ({3.5 * \nodesize}, {-0.5 * \nodesize - 0.1}) node[left = {2*\nodesize}, below = {0.1*\nodesize}]{$\mathbf{x}^{\beta}$};

        \foreach \name / \x in {0,2}
            \path[xshift=-2.25cm]
                node[neuron] (t-\name) at ({\x * \nodesize * 1}, {3 * \nodesize}) {};
         \foreach \name / \x in {1}
            \path[xshift=-2.25cm]
                node[dark_neuron] (t-\name) at ({\x * \nodesize * 1}, {3 * \nodesize}) {};
        \draw[xshift=-2.25cm, yshift={3 * \nodesize}, line width = 0.5mm] ({-0.5 * \nodesize}, {0.5 * \nodesize + 0.1}) rectangle ({2.5 * \nodesize}, {-0.5 * \nodesize - 0.1}) node[left = {1.5 * \nodesize}, above = \nodesize] {$\mathbf{h}$}; 

        \foreach \name / \x in {0,...,5}
            \path[xshift=-2cm]
                node[neuron] (w-\name) at ({-\x * \nodesize * 1}, 0) {};
        \draw[xshift=-2cm, line width = 0.5mm] ({0.5 * \nodesize}, {0.5 * \nodesize + 0.1}) rectangle ({-5.5 * \nodesize}, {-0.5 * \nodesize - 0.1}) node[right = {3 * \nodesize}, below = {0.1*\nodesize}]{$\mathbf{x}^{\alpha}$};

         \foreach \name / \x in {0,...,5}
                \draw[line width =0.7mm,>=latex,->,draw=black!100] ({-2cm - \x * \nodesize}, {\nodesize * 0.5}) -- ({-2.1cm + \nodesize * 0.7}, {2.5 * \nodesize});


        \foreach \name / \x in {0,...,2}
            \draw[line width =0.7mm,>=latex,->,draw=black!100] ({-0.5cm + 1 * \nodesize}, {\nodesize * 0.5}) -- ({-2.25cm + \x * \nodesize * 1}, {2.5 * \nodesize});
    \end{tikzpicture}
    }
    \vspace{-.6cm}
    \caption{(b)~SpikePLSI and SpikeLDA}
    \end{subfigure}

    \begin{subfigure}{0.5\textwidth}
    \resizebox{0.9\textwidth}{!}{
    \hspace{2cm}
    \begin{tikzpicture}[node distance=\nodedist]
        \def\nodesize{15pt};
        \tikzstyle{every pin edge}=[<-,shorten <=2pt]
        \tikzstyle{neuron}=[circle, draw=black!100, line width=0.5mm, inner sep=0pt, minimum size=\nodesize]
        \tikzstyle{dark_neuron}=[circle, fill=black!100, draw=black!100, line width=0.5mm, inner sep=0.5pt, minimum size=\nodesize]
        \tikzstyle{grey_neuron}=[circle, draw=black!100, line width=1.75mm, inner sep=0.5mm, minimum size=\nodesize - 1mm]
        \tikzstyle{dot}=[circle, draw=black!100, fill=black!100, line width=0.5pt, inner sep=0.5pt, minimum size=0.1mm]
        \tikzstyle{annot} = [text width=10em, text centered]

        \node[dark_neuron] at ({-1 * \nodesize * 1}, {3 * \nodesize}){}; \node[] at ({2.9 * \nodesize * 1}, {3 * \nodesize}) {\LARGE : spiking neuron};
        \node[grey_neuron] at ({7 * \nodesize * 1}, {3 * \nodesize}){}; \node[] at ({11.4 * \nodesize * 1}, {3 * \nodesize}) {\LARGE: excitatory neuron};
        \draw[line width =0.7mm,>=latex,->,draw=black!100]  ({15.8 * \nodesize * 1}, {3 * \nodesize}) -- ({17.5 * \nodesize * 1}, {3 * \nodesize});
        \node[] at ({21.2 * \nodesize * 1}, {3 * \nodesize}) {\LARGE: modified synapse};
    \end{tikzpicture}
    }
    \end{subfigure}
    \caption{The synapses are modified at a spike timing}
    \label{fig.synapse update}
\end{figure}
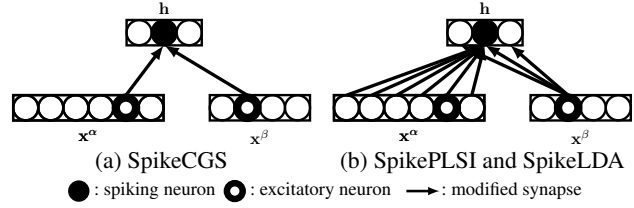

\section{Derivation of Eq.~(\ref{equ.semi ELBO})}\label{appendix.elbo}
In the vanilla LDA model (where there is no prior defined on the global parameter $\boldsymbol\Phi$), the semi-collapsed likelihood can be re-written as the evidence lower bound plus the KL divergence between a variational distribution $q(\mathbf{Z})$ and the exact posterior:
\begin{align*}
    \log p(\mathbf{W}|\boldsymbol\Phi,\boldsymbol\lambda) &= \mathbb{E}_{q(\mathbf{Z})}\log p(\mathbf{W}|\mathbf{Z};\boldsymbol\Phi,\boldsymbol\lambda)+\mathbb{E}_{q(\mathbf{Z})}\log \frac{p(\mathbf{Z}|\boldsymbol\lambda)}{q(\mathbf{Z})} + \text{KL}(q(\mathbf{Z})||p(\mathbf{Z}|\mathbf{W};\boldsymbol\Phi,\boldsymbol\lambda)).
\end{align*}
If $q(\mathbf{Z})=p(\mathbf{Z}|\mathbf{W};\boldsymbol\Phi,\boldsymbol\lambda)$, e.g., by CGS, the KL term is eliminated.
Moreover the second term on the r.h.s. is a constant $C$ w.r.t. $\boldsymbol\Phi$, the variable we optimize.
Therefore we have Eq.~(\ref{equ.semi ELBO}):
\begin{align*}
    \log p(\mathbf{W}|\boldsymbol\Phi,\boldsymbol\lambda)= &\mathbb{E}_{q(\mathbf{Z})}\log p(\mathbf{W}|\mathbf{Z};\boldsymbol\Phi,\boldsymbol\lambda)+ C
\end{align*}

\section{The complete semi-SpikeLDA algorithm}
The complete semi-SpikeLDA algorithm is summarized in Alg.~\ref{alg.semiSpikeLDA}.
\begin{algorithm}[h!]
    \caption{Semi-SpikeLDA, where $\tau_1(x) = \log( \exp x -1)$ and $\tau_2(x) = \log( \exp x +1)$.}
    \begin{algorithmic}[1]
        \Function{CGS in semi-SpikeLDA}{A mini-batch $\hat{D}$ of documents, the number of semi-CGS iterations $T$}
        \For{$d\in \hat{D}$}
        \State Initialize $\mathbf{z}_d^0,$ and $M_{zd}^{\beta} = \log(C_{z,d}+\varphi_z), \forall z$
        \For{Sample $s\in\{1,...,2T\}$}
        \For{Location $i\in\{1,...,N_d\}$}
        \State Negative phase: $M_{z_{di}^{s-1}d}^{\beta} = \tau_1(M_{z_{di}^{s-1}d}^{\beta})$\label{equ.negative phase}
        \State Resample:~$z_{di}^{s}$ = \Call{InferenceDynamics}{$w_{di},d$}\label{alg.line.snn resample}
        \State Positive phase: $M_{{z_{di}^{s}}d}^{\beta} = \tau_2(M_{{z_{di}^{s}}d}^{\beta})$ \label{equ.positive phase}
        \EndFor
        \EndFor
        \EndFor
        \State\Return{$\hat{N}_{z,w}\!\triangleq\!\sum_{s=T+1}^{2T}\sum_{d=1}^D\sum_{i=1}^{N_d}\mathbb{I}(w_{di}\!=\!w,z_{di}^s\!=\!z),\forall w,z$ and $\hat{N}_{z}\!\triangleq\!\sum_{w=1}^V\hat{N}_{z,w},\forall z$.}
        \EndFunction
    \end{algorithmic}

    \begin{algorithmic}[1]
        \Function{Semi-SpikeLDA}{A corpus, the number of semi-CGS iterations $T$}
        \Repeat
        \State Sample a mini-batch of document $\hat{D}$
        \State $\{\hat{N}_{z,w}, \hat{N}_z,\forall w,z\}$ = \Call{CGS in semi-SpikeLDA}{$\hat{D}$, $T$}
        \State Update \begin{align}
    M_{zw}^{\alpha} &\leftarrow M_{zw}^{\alpha} + \eta_t \frac{1}{|\hat{D}|T}\Big[\hat{N}_{z,w}\exp(-M_{zw}^{\alpha}) - \hat{N}_{z}\Big],\forall w,z\nonumber
\end{align}
        \Until{A pre-specified number of iterations is reached.}
        \EndFunction
    \end{algorithmic}
    \label{alg.semiSpikeLDA}
\end{algorithm}

\section{Online SNN implementation of pLSI}
pLSI~\cite{Hofmann1999} is a special case of LDA where there is no prior.
So our SpikeLDA algorithm can be extended to train pLSI.
\subsection{Preliminary of pLSI}
Probabilistic Latent Semantic Indexing (pLSI)~\cite{Hofmann1999} is the most fundamental topic model, which assumes a generative process for the corpus $\mathbf{W}$:
\begin{enumerate}\vspace{-.1cm}
    \item[]for each document $d=1,...,D$,
        \begin{enumerate}
            \item[] for each position in the document, $i=1,...,N_d$,
                \begin{enumerate}
                    \item[] draw a topic assignment $z_{di}\sim\text{Multi}(\boldsymbol\theta_d)$;
                    \item[] draw a word $w_{di}\sim\text{Multi}(\boldsymbol\phi_{z_{di}})$,
                \end{enumerate}\vspace{-.1cm}
        \end{enumerate}
\end{enumerate}
where $\boldsymbol\theta_d$ is a $K$-dimensional topic-mixing proportion vector of document $d$; $K$ is the pre-specified number of topics; $\boldsymbol\phi_k$ is a $V$-dimensional topic distribution vector for topic $k$; and $\text{Multi}(\cdot)$ denotes the Multinomial distribution.

Maximum-likelihood (ML) is a standard parameter estimation method for pLSI.
According to Eq.~(\ref{equ.L stochastic form}) in the main text, the ML problem is
\begin{align}
    \max_{\boldsymbol\Phi,\boldsymbol\Theta}~ \mathbb{E}_{\pi(w,d)} \big[ \log p(w|d;\boldsymbol\Phi,\boldsymbol\Theta) \big],\label{equ.ML stochastic problem}
\end{align}
In the new formulation of Eq.~(\ref{equ.ML stochastic problem}), $\pi(w,d)$ an also be the unknown environment distribution.

\subsection{The SNN algorithm}
The SNN algorithm to train pLSI, called SpikePLSI, is the same as the SpikeLDA algorithm Alg.~\ref{alg.online algorithm} except that the stochastic update rule is replaced by
\begin{align}
    M_{zw}^{\alpha} &\leftarrow M_{zw}^{\alpha} + \eta_t h_z\left.(x_w^{\alpha}\exp(-M_{zw}^{\alpha})-1\right.),\forall w,z,\nonumber\\
    M_{zd}^{\beta} &\leftarrow M_{zd}^{\beta} + \eta_t x_d^{\beta}\left.(h_z\exp(-M_{zd}^{\beta})-1\right.),\forall d,z.\label{equ.ML dicrete dynamics}
\end{align}
In the following, we present the theoretical results.
\subsubsection{Conceived learning problem}
The problem, Def.~\ref{def.SpikePLSI}, resembles the online MLE.
\begin{definition}\label{def.SpikePLSI}
A SpikePLSI problem is defined as
\begin{align}
    \max_{\mathbf{M}}&~~\mathbb{E}_{\pi(w,d)}\log p(x_w^{\alpha}=1|x_d^{\beta}=1;\mathbf{M}),\nonumber\\
    \text{s.t. }&\zeta(\mathbf{M}_{z\cdot}^{\alpha}) = 1,\forall z,\quad \zeta(\mathbf{M}_{\cdot d}^{\beta}) = 1,\forall d.\label{equ.ML SNN problem}
\end{align}
\end{definition}
We prove its equivalence with problems (\ref{equ.ML stochastic problem}) in Lemma~\ref{prop.ML equivalence}.
\begin{lemma}\label{prop.ML equivalence} 
\vspace{-.1cm}
The SpikePLSI problem, Def.~\ref{def.SpikePLSI}, is equivalent with the pLSI problem (\ref{equ.ML stochastic problem}).
\end{lemma}
\begin{proof}
\vspace{-.2cm}
According to Lemma~\ref{prop.re-parametrization}, subject to the normalization constrains and after summing $z$ out and taking the expectation over $\pi(w,d)$, we have:
\begin{align}
    &\mathbb{E}_{\pi(w,d)}\!\log\!p(x_w^{\alpha}\!=\!1|x_d^{\beta}\!=\!1;\mathbf{M})\!=\!\mathbb{E}_{\pi(w,d)}\!\log \!p(w|d;\!\boldsymbol\Phi,\!\boldsymbol\Theta).\nonumber
\end{align}
Then the two optimization problems are equivalent because they differ from a change of variables and the transformation is one-to-one~\cite{Boyd2004}.
\end{proof}
Finally, we prove that the update rules in the SpikePLSI algorithm is a stochastic parameter update rule whose ML-ODE solves the conceived problem of Def.~\ref{def.SpikePLSI}.
\begin{theorem}\label{prop.Ml stochastic update} (Proof in Appendix \ref{proof.ML dynamics})
In Alg.~\ref{alg.online algorithm},
the ML-ODE of the update rule Eq.~(\ref{equ.ML dicrete dynamics}) solve the SpikePLSI problem.
So does the stochastic update rule Eq.~(\ref{equ.ML dicrete dynamics}). This algorithm is called the SpikePLSI algorithm.
\end{theorem}

\section{Details of the experiments}
\subsection{Datasets}
The datasets are standard datasets from the UCI machine learning repository.
KOS is a collection of blog entries on US current events; NIPS contains full papers of NIPS from $1990$ to $2002$; Enron is a collection of emails from about $150$ users of Enron; NYTimes and Pubmed consist of news articles and biomedical literature abstracts.
Each dataset is randomly split into the training set and testing set with ratio $9:1$.
And each document in the testing set is split into two halves: one half belongs to a observed testing set, and the other half belongs to a holdout testing set.
Then we use the fold-in method in~\cite{Asuncion2009} to calculate the perplexity.
\begin{table}[h!]
\centering
\begin{tabular}{|l|l l l l|}\hline
Dataset&   $D$&   $N$&  $V$&   $T/D$\\ \hline
    KOS&  3.4K&  467K&  69K&   136\\ \hline
   NIPS&  1.5K&  1.9M&  12K&   1266\\ \hline
  Enron&   39K&  6.4M&  28K&   160\\ \hline
  Nytimes& 300K& 100M& 102K&   332\\ \hline
  PubMed& 8.2M&  738M&  141K&   90\\\hline
    \end{tabular}
\caption{Statistics of the datasets}
\label{tlb.dataset}
\vspace{-10pt}
\end{table}

\subsection{Network initialization}
All synapses are randomly sampled from $N(1,1)$, except that each synapse of the $\mathbf{M}^{\beta}$ is initialized to $\log(1/K)$ in SpikePLSI.

\subsection{Adaptive step size schedule}
\subsubsection{For SpikePLSI and ed-SpikeLDA}
For all parameter $\mathbf{M}$ in SpikePLSI, we use variance tracking~\cite{Nessler2013}.
For $\mathbf{M}^{\alpha}$ of ed-SpikeLDA, we use variance tracking too; but for its $\mathbf{M}^{\beta}$, we use AdaGrad.
To avoid the annealing step size problem of AdaGrad \cite{Zeiler2012}, we amplify the step size by a factor $c$ at the end of every iteration.
For KOS, $c=1$. For 20NG, $c=2$ when $\varphi=1.01$, $c=4$ when $\lambda = 1.05$. For Enron, $c=24$ when $\lambda = 1.01$, $c=72$ when $\lambda = 1.05$, $c=2.7$ when $\lambda = 1.1$.
\subsubsection{For du-SpikeLDA and semi-SpikeLDA}
Tab.~\ref{tlb.nips},~\ref{tlb.nytimes} and ~\ref{tlb.pubmed} list the parameters used in the algorithms.
For the du-SpikeLDA, $\eta_{\alpha}$ means the step size for $\mathbf{M}^{\alpha}$ and $\eta_{\beta}$ means the step size for $\mathbf{M}^{\beta}$.

\begin{table}[h!]
    \centering
    \begin{tabular}{|c|c|c|c|c|c|}\hline
    Algorithm & $\varphi$ & $\lambda$ & Adaptive step size method & Local iterations\\\hline
    SVI & 0.51 & 0.51,~0.6,~0.7,~1,~1.25,~1.5 & $a\cdot(1+t/b)^{-c},a=0.1,b=10000,c=0.7$& 5\\\hline
    SGMCMC & 0.01 & 0.01,~0.1,~0.2,~0.5,~0.75,~1 & $a\cdot(1+t/b)^{-c},a=0.05,b=1000,c=0.7$& 10\\\hline
    semi-SpikeLDA & $\backslash$ & 0.01,~0.1,~0.2,~0.5,~0.75,~1 & RMSP with $\eta = 0.5$ & 10\\\hline
    du-SpikeLDA & $\backslash$ & 1.01,~1.1,~1.2,~1.5,~1.75,~2 & RMSP with $\eta_{\alpha}=0.5$, RMSP with $\eta_{\beta}=1$ & 30\\\hline
    \end{tabular}
    \caption{Dataset: NIPS, $K = 50$, global iterations: 1000}
    \label{tlb.nips}

    \begin{tabular}{|c|c|c|c|c|}\hline
    Algorithm & $\varphi$ & $\lambda$ & Adaptive step size method & Local iterations \\\hline
    SVI & 0.51 & 0.51,~0.6,~0.7,~1,~1.25,~1.5 & $a\cdot(1+t/b)^{-c},a=0.01,b=10000,c=0.6$& 30 \\\hline
    SGMCMC & 0.01 & 0.01,~0.1,~0.2,~0.5,~0.75,~1 & $a\cdot(1+t/b)^{-c},a=0.005,b=1000,c=0.7$& 30 \\\hline
    semi-SpikeLDA & $\backslash$ & 0.01,~0.1,~0.2,~0.5,~0.75,~1 & RMSP with $\eta = 0.1$ & 30 \\\hline
    du-SpikeLDA & $\backslash$ & 1.01,~1.1,~1.2,~1.5,~1.75,~2 & RMSP with $\eta_{\alpha}=0.05$, RMSP with $\eta_{\beta}=1$ & 30 \\\hline
    \end{tabular}
    \caption{Dataset: Nytimes, $K = 50$, global iterations: 3000, mini-batch size: 1000}
    \label{tlb.nytimes}

    \begin{tabular}{|c|c|c|c|c|}\hline
    Algorithm & $\varphi$ & $\lambda$ & Adaptive step size method & Local iterations \\\hline
    SVI & 0.51 & 0.51,~0.6,~0.7,~1,~1.25,~1.5 & $a\cdot(1+t/b)^{-c},a=0.05,b=10000,c=0.6$& 30 \\\hline
    SGMCMC & 0.01 & 0.01,~0.1,~0.2,~0.5,~0.75,~1 & $a\cdot(1+t/b)^{-c},a=0.0001,b=1000,c=0.7$& 30 \\\hline
    semi-SpikeLDA & $\backslash$ & 0.01,~0.1,~0.2,~0.5,~0.75,~1 & RMSP with $\eta = 0.05$ & 30 \\\hline
    du-SpikeLDA & $\backslash$ & 1.01,~1.1,~1.2,~1.5,~1.75,~2 & RMSP with $\eta_{\alpha}=0.05$, RMSP with $\eta_{\beta}=1$ & 30 \\\hline
    \end{tabular}
    \caption{Dataset: PubMed, $K = 50$, global iterations: 5000, mini-batch size: 1000}
    \label{tlb.pubmed}
\end{table}

\subsection{NIPS results}
The NIPS results are shown in Fig.~\ref{fig.NIPS result 1} and~\ref{fig.NIPS result 2}.
\begin{figure}[h!]
    \centering
    \includegraphics[width=0.3\textwidth]{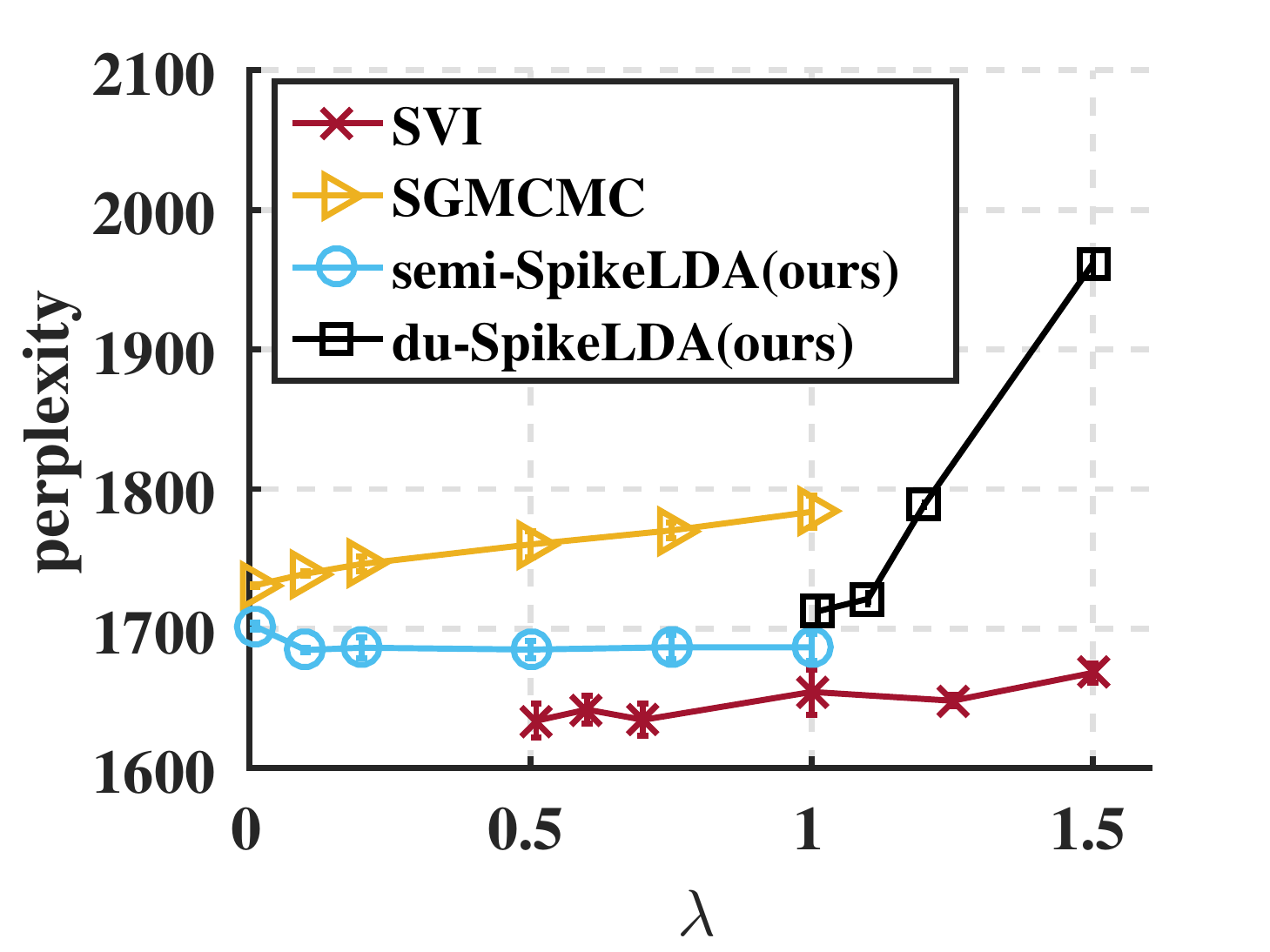}
    \caption{NIPS result. Impact of $\lambda$. $K=50$. Iterations number: $1000$}
    \label{fig.NIPS result 1}
\end{figure}
\begin{figure}[h!]
    \centering
        \includegraphics[width=0.3\textwidth]{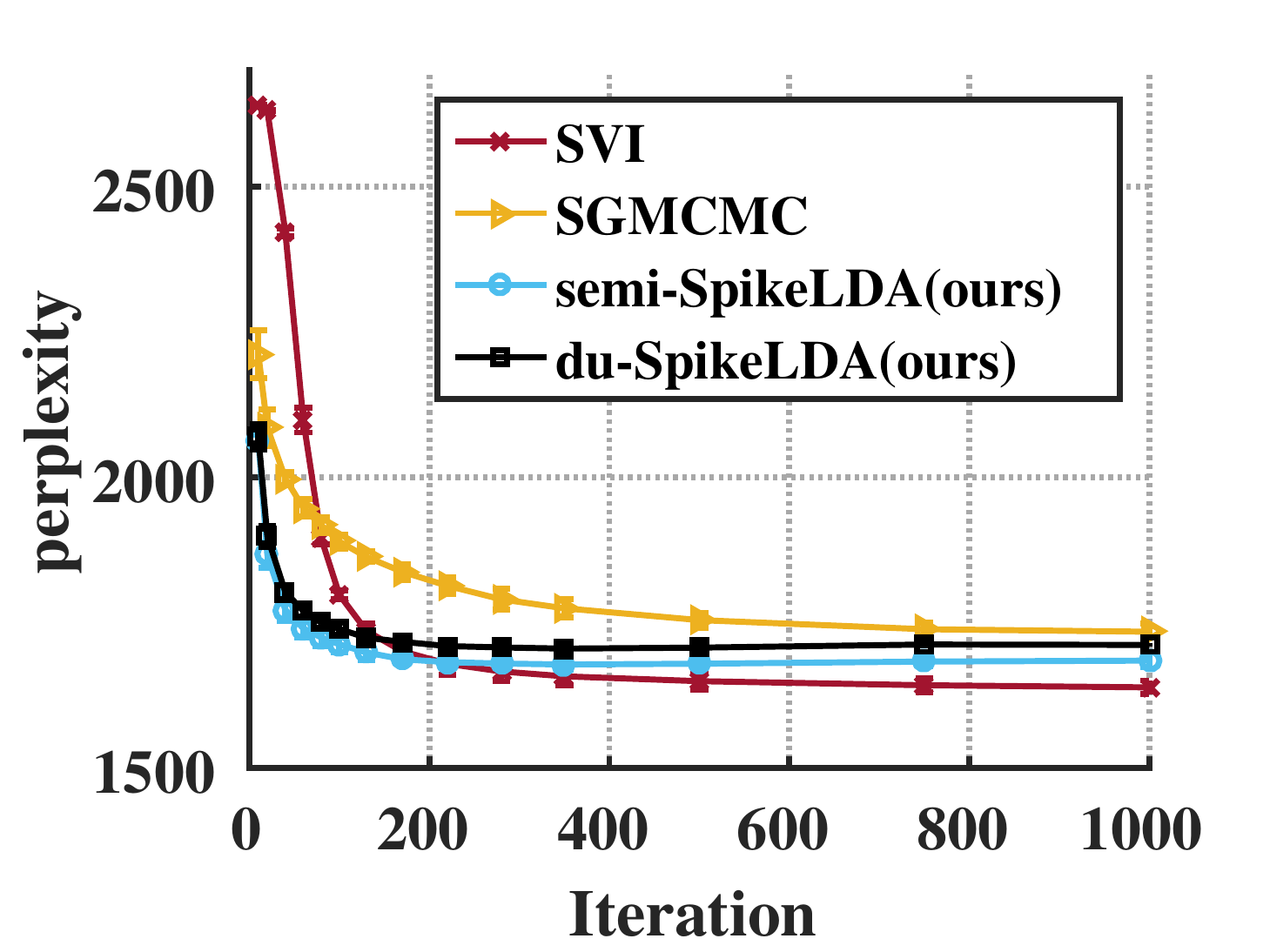}
        \caption{NIPS result. Convergence results. $K=50$.}
        \label{fig.NIPS result 2}
\end{figure}

\section{Proofs to the theoretical results}
\subsection{Proof of lemma \ref{prop.re-parametrization}}
If the following conditions holds,
\begin{enumerate}
\item $\zeta(\mathbf{M}_{z\cdot}^{\alpha})\!=\!1,\forall z$ and $ \zeta(\mathbf{M}_{\cdot d}^{\beta})\!=\!\kappa
,\forall d$,
where $\kappa$ is some constant;
\item the variables and parameters are related by Tab.~\ref{tlb.Parameter relation}(a,c);
\end{enumerate}
then Eq.~(\ref{equ.SNN ML}) equals the complete likelihood that a word $w$ in document $d$ is assigned the topic $z$ in LDA:
\begin{align*}
p(x_w^{\alpha} \!=\! 1, h_z \!=\! 1|x_d^{\beta} \!=\! 1;\mathbf{M}) = p(w,z|d;\boldsymbol\Phi,\boldsymbol\Theta),
\end{align*}

Moreover, the conditional distribution of the topic assignment for a token is:
\begin{align*}
    p(h_z=1|\mathbf{x}^{\alpha},\mathbf{x}^{\beta};\mathbf{M})&\propto\exp(u_z),\forall z,
\end{align*}
where $u_z = \mathbf{M}_{z\cdot}^{\alpha} \cdot \mathbf{x}^{\alpha}  + \mathbf{M}_{z\cdot}^{\beta}\cdot\mathbf{x}^{\beta}$ is the weighted sum of the input $\mathbf{x}^{\alpha},\mathbf{x}^{\beta}$.
\begin{num_proof}\label{proof.re-parameterization}
The probabilistic model defined on the SNN is Eq.~(\ref{equ.SNN ML}), i.e.,
\begin{align}
 p(x_w^{\alpha} &= 1, h_z = 1|x_d^{\beta} = 1;\mathbf{M}^{\alpha},\mathbf{M}^{\beta})
 = \exp\big[\mathbf{M}_{z\cdot}^{\alpha}\cdot\mathbf{x}^{\alpha}  +  \mathbf{M}_{z\cdot}^{\beta}\cdot\mathbf{x}^{\beta}  - A(\mathbf{M}_{z\cdot}^{\alpha}, \mathbf{M}_{\cdot d}^{\beta})\big].\nonumber
\end{align}
Under the assumption, $\zeta(\mathbf{M}_{z\cdot}^{\alpha})=\sum_{w=1}^V\exp(M_{zw}^{\alpha})=1,\forall k,\zeta(\mathbf{M}_{\cdot d}^{\beta})=\sum_{z=1}^K\exp(M_{zd}^{\beta})=\kappa,\forall d$, it is
\begin{align}
    p(x_w^{\alpha} = 1, h_z = 1|x_d^{\beta} = 1;\mathbf{M}^{\alpha}, \mathbf{M}^{\beta})
    = \exp\left.(M_{zw}^{\alpha}+M_{zd}^{\beta} - \log \kappa\right.). \label{equ.proof re-parametrization 1}
\end{align}
Under the transformations in Table \ref{tlb.Parameter relation}(c),
\begin{align}
    \phi_{zw} &= \exp M_{zw}^{\alpha},\nonumber\\
    \theta_{dz} &=\exp\left.(M_{zd}^{\beta} - \log \kappa\right.),
\end{align}
Eq.~(\ref{equ.proof re-parametrization 1}) is re-written as
\begin{align}
    p(x_w^{\alpha} = 1, h_z = 1|x_d^{\beta} = 1;\mathbf{M}^{\alpha}, \mathbf{M}^{\beta}) &= \phi_{zw}\theta_{dz} = p(w,z|d;\boldsymbol\Phi,\boldsymbol\Theta),
\end{align}
and $\theta_{dz}$ and $\phi_{zw}$ are normalized:
\begin{align}
    \sum_{w=1}^V\phi_{zw} &= \sum_{w=1}^V\exp(M_{zw}^{\alpha}) = 1,\nonumber\\
    \sum_{z=1}^K\theta_{dz} &= \sum_{z=1}^K\exp(M_{dz}^{\beta} - \log \kappa) = 1,
\end{align}
where $p(w,z|d;\boldsymbol\Phi,\boldsymbol\Theta)$ is the definition of the likelihood of a word $w$ assigned to topic $z$ in document $d$ in pLSI.
The transformations are one-to-one, since $\exp(\cdot)$ is monotone increasing, and one-hot representation is obviously invertible.

$p(h_z=1|\mathbf{x}^{\alpha},\mathbf{x}^{\beta};\mathbf{M})\propto\exp(u_z)$ means that the posterior is proportional to the joint.
Because $A(\mathbf{M}_{z\cdot}^{\alpha}, \mathbf{M}_{\cdot d}^{\beta})$ is a constant under conditions $\zeta(\mathbf{M}_{z\cdot}^{\alpha})=1,\forall z$ and $ \zeta(\mathbf{M}_{\cdot d}^{\beta})=\kappa
,\forall d$, the result follows.
\qed
\end{num_proof}
\leavevmode\newline
\subsection{Proof of lemma \ref{prop.ML equivalence}}
The SpikePLSI problem, Def.~\ref{def.SpikePLSI}, is equivalent with the pLSI problem (\ref{equ.ML stochastic problem}).
\begin{num_proof}\label{proof.ML equivalence}
According to lemma \ref{prop.re-parametrization}, under the condition $\zeta(\mathbf{M}_{z\cdot}^{\alpha}) =1,\forall z, \zeta(\mathbf{M}_{\cdot d}^{\beta})=1, \forall d$, the following equality holds:
\begin{align*}
    p(x_w^{\alpha} = 1, h_z = 1|x_d^{\beta} = 1;\mathbf{M}^{\alpha}, \mathbf{M}^{\beta}) = p(w,z|d;\boldsymbol\Phi,\boldsymbol\Theta).
\end{align*}
Summing the l.h.s. over $h_z=1,\forall z$ and summing the r.h.s. over $z$, and taking the expectation over $\pi(w,d)$, the following equality holds:
\begin{align*}
    \mathbb{E}_{\pi(w,d)}\log p(x_w^{\alpha} = 1|x_d^{\beta} = 1;\mathbf{M}^{\alpha}, \mathbf{M}^{\beta}) = \mathbb{E}_{\pi(w,d)}\log p(w|d;\boldsymbol\Phi,\boldsymbol\Theta).
\end{align*}
The two optimization problems,
\begin{align*}
    \max_{\mathbf{M}^{\alpha},\mathbf{M}^{\beta}}\mathbb{E}_{\pi(w,d)}\log p(x_w^{\alpha} = 1|x_d^{\beta} = 1;\mathbf{M}^{\alpha}, \mathbf{M}^{\beta})
\end{align*}
and
\begin{align*}
    \max_{\boldsymbol\Phi,\boldsymbol\Theta}\mathbb{E}_{\pi(w,d)}\log p(w|d;\boldsymbol\Phi,\boldsymbol\Theta)
\end{align*}
are equivalent because they differ from a change of variables and the transformation is one-to-one \cite{Boyd2004}.\qed
\end{num_proof}

\subsection{Proof of lemma \ref{prop.MAP equivalence}}
The SpikeLDA problem, Def.~\ref{def.SpikeLDA}, is equivalent to the LDA problem (\ref{equ.MAP stochastic problem}), when $\kappa\!=\!\sum_{z=1}^K (\lambda_z\!-\!1); \forall z, \lambda_z\!\geq\!0$.
\begin{num_proof}\label{proof.MAP equivalence}
According to lemma \ref{prop.re-parametrization}, under the condition $\zeta(\mathbf{M}_{z\cdot}^{\alpha}) =1,\forall z, \zeta(\mathbf{M}_{\cdot d}^{\beta})=\kappa=\sum_{z=1}^K (\lambda_z - 1), \forall d$, we have
\begin{align}
    \mathbb{E}_{\pi(w,d)}\log p(x_w^{\alpha}=1|x_d^{\beta}=1;\mathbf{M}^{\alpha}, \mathbf{M}^{\beta})= \mathbb{E}_{\pi(w,d)}\log p(w|d;\boldsymbol\Phi,\boldsymbol\Theta).
\end{align}
Now we only need to prove $p(\mathbf{M}^{\beta};\boldsymbol\lambda) \propto p(\boldsymbol\Theta;\boldsymbol\lambda)$. The following derivations make it clear:
\begin{align*}
    p(\mathbf{M}^{\beta}_{\cdot d};\boldsymbol\lambda) &= \prod_{z=1}^K p(M_{zd}^{\beta};\boldsymbol\lambda)\nonumber\\
    &\propto \prod_{z=1}^K \exp((\lambda_z - 1) M_{zd}^{\beta})\exp(-e^{M_{zd}^{\beta}})\nonumber\\
    &= \exp(-\sum_{z=1}^K e^{M_{zd}^{\beta}}) \prod_{z=1}^K \exp((\lambda_z - 1) M_{zd}^{\beta})\nonumber\\
    &= \exp(-\kappa) \prod_{z=1}^K \exp((\lambda_z - 1) M_{zd}^{\beta})\nonumber\\
    &\propto \prod_{z=1}^K \exp((\lambda_z - 1) M_{zd}^{\beta})\nonumber\\
    &= \prod_{z=1}^K \theta_{dz}^{\lambda_z - 1}\nonumber\\
    &\propto \text{Dir}(\boldsymbol\theta_d;\boldsymbol\lambda)\nonumber\\
    &= p(\boldsymbol\theta_{d};\boldsymbol\lambda),\forall d.
\end{align*}
Therefore, $\log p(\mathbf{M}^{\beta}_{\cdot d};\boldsymbol\lambda) = \log p(\boldsymbol\theta_{d};\boldsymbol\lambda) + C,\forall d$, where $C$ is a constant respective to $\mathbf{M}$ and $\boldsymbol\Theta$.
The third equality is from the constraint $\sum_{z=1}^K\exp(M_{zd}^{\beta})=\kappa,\forall d$.
The two optimization problems,
\begin{align*}
    \max_{\mathbf{M}^{\alpha},\mathbf{M}^{\beta}}\mathbb{E}_{\pi(w,d)}\Big[\log p(x_w^{\alpha} = 1|x_d^{\beta} = 1;\mathbf{M}^{\alpha}, \mathbf{M}^{\beta})+\frac{1}{N_d}\log p(\mathbf{M}_{\cdot d}^{\beta}|\boldsymbol\lambda)\Big]
\end{align*}
and
\begin{align*}
    \max_{\boldsymbol\Phi,\boldsymbol\Theta}\mathbb{E}_{\pi(w,d)}\Big[\log p(w|d;\boldsymbol\Phi,\boldsymbol\Theta)+\frac{1}{N_d}\log p(\boldsymbol\theta_d|\boldsymbol\lambda)\Big]
\end{align*}
are equivalent because they differ from a change of variables and the transformation is one-to-one \cite{Boyd2004}.\qed
\end{num_proof}

\subsection{Proofs of Thm. \ref{prop.Ml stochastic update}, Thm. \ref{prop.MAP stochastic update} and Thm.~\ref{prop.semi stochastic update}}
The proofs to the three theorems share common structures.
Remember that the theorems are about that, for a particle $\mathbf{M}$ randomly initialized in the space, if its dynamics in the space follows the "mean limit" ordinary differential equations (ML-ODEs) of the stochastic update rules, the convergence points of the particle is the same as the local maximum of the corresponding conceived constrained optimization problems.
We separate the proofs into two part.
The first part~\ref{proof.convergence to manifold} proves that the particle converges to the manifold defined by the constrains.
The second part~\ref{proof.convergence to maximum} proves that the particle converges to the local maximum on the manifold.

Notations are introduced.
We denote the ML-ODE of any stochastic update rule as
\begin{align}
    \frac{d}{ds}M_{zw}^{\alpha} &= g^{\alpha}(M_{zw}^{\alpha}),~~~ \forall z,w,\nonumber\\
    \frac{d}{ds}M_{zd}^{\beta} &= g^{\beta}(M_{zd}^{\beta}),~~~ \forall z,d.\label{equ.proof ML ML-ODE}
\end{align}
Let $g(\mathbf{M})$ collectively denote the temporal derivatives of all components in $\mathbf{M}$, i.e. $\frac{d}{ds}\mathbf{M}$.
Now we define some sets in the $\mathbf{M}$ space to describe the convergence behavior. The positive limit set of a set $B$, describing the dynamics when $s\rightarrow\infty$ and initialization $\mathbf{M}(0)\in B$, is defined as
\begin{align}
    F(B) \triangleq \lim_{s\rightarrow\infty}\cup_{\mathbf{M}\in B}\{\mathbf{M}(s'),s'>s:\mathbf{M}(0) = \mathbf{M}\}.\label{equ.proof ML continuous 6}
\end{align}
The set of stationary points in $B$ is defined as:
\begin{align}
    S(B) \triangleq \{\mathbf{M}\in B : g(\mathbf{M}) = \mathbf{0}\}.\label{equ.proof ML continuous 7}
\end{align}
We define a set of function $\zeta_{z}^{\alpha}(\mathbf{M}),\zeta_{d}^{\beta}(\mathbf{M}),z=1,...,K,d=1,...,D$ to represent the deviation of $\mathbf{M}$ from each of the constraints:
\begin{align*}
    \zeta_{z}^{\alpha}(\mathbf{M}) = \sum_{w=1}^V \exp(M_{zw}^{\alpha}) - 1,\forall z\nonumber\\
    \zeta_{d}^{\beta}(\mathbf{M}) = \sum_{z=1}^K \exp(M_{zd}^{\beta}) - \kappa,\forall d.
\end{align*}
Therefore, the manifold defined by the constraints is
\begin{align}
H\triangleq \{\mathbf{M} : \zeta_z^{\alpha}(\mathbf{M}) = 0, \forall z, \zeta_d^{\beta}(\mathbf{M}) = 0, \forall d\}.\label{equ.proof ML manifold def}
\end{align}
For any initial state $\mathbf{M}(0)$, we can define a set
\begin{align}
     C\triangleq \{\mathbf{M} : |\zeta_z^{\alpha}(\mathbf{M})| \leq |\zeta_z^{\alpha}(\mathbf{M}(0))|,\forall z, |\zeta_d^{\beta}(\mathbf{M})|\leq |\zeta_d^{\beta}(\mathbf{M}(0))|,\forall d\}.\label{equ.proof ML close set def}
\end{align}
It is straightforward that $H\subseteq C$ by the definitions Eq.~(\ref{equ.proof ML manifold def}) and (\ref{equ.proof ML close set def}).
\subsubsection{Convergence to the manifold defined by the constrains}
Let $\hat{\pi}(w,d,z) = \pi(w,d)q(z|w,d)$ denote the joint probability that a token $(w,d)$ is observed and the topic $z$ is assigned to it where $q(z|w,d)$ is arbitrary stationary distribution, e.g., the $p(h_z|x_w^{\alpha},x_d^{\beta};\mathbf{M}^{\alpha},\mathbf{M}^{\beta})$.
The specific form of $q(z|w,d)$ doesn't change the proof.
Now we prove that the convergence to manifold results for SpikePLSI, SpikeLDA and semi-SpikeLDA respectively.
\begin{num_proof}\label{proof.ML dynamics}\label{proof.convergence to manifold}
~~
\subsubsection{SpikePLSI}
In SpikePLSI algorithm, the $\kappa=1$ and the ML-ODE is
\begin{align}
    \frac{d}{ds}M_{zw}^{\alpha} &= g^{\alpha}(M_{zw}^{\alpha})\triangleq \mathbb{E}_{\hat{\pi}(w^*,d^*,z^*)}\Big\{\mathbb{I}(z=z^*)\big[\mathbb{I}(w=w^*)\exp(-M_{zw}^{\alpha}) - 1\big]\Big\},\forall z,w,\nonumber\\
    \frac{d}{ds}M_{zd}^{\beta} &= g^{\beta}_{ml}(M_{zd}^{\beta})\triangleq \mathbb{E}_{\hat{\pi}(w^*,d^*,z^*)}\Big\{\mathbb{I}(d=d^*)\big[\mathbb{I}(z=z^*)\exp(-M_{zd}^{\beta}) - 1\big]\Big\},\forall z,d.\label{equ.proof ML ML-ODE}
\end{align}
The temporal evolutions of deviations from the manifold are:
\begin{align}
    \frac{d}{ds}\zeta_{z}^{\alpha}(\mathbf{M}(s)) &= \nabla \zeta_{z}^{\alpha}(\mathbf{M}(s))\cdot g(\mathbf{M}(s))\nonumber\\
    &= \sum_{w=1}^V\exp(M_{zw}^{\alpha}) \mathbb{E}_{\hat{\pi}(w^*,d^*,z^*)}\Big\{\mathbb{I}(z=z^*)\big[\mathbb{I}(w=w^*)\exp(-M_{zw}^{\alpha}) - 1\big]\Big\}\nonumber\\
    &= \sum_{w=1}^V\mathbb{E}_{\hat{\pi}(w^*,d^*,z^*)}\Big\{\mathbb{I}(z=z^*)\big[\mathbb{I}(w=w^*) - \exp(M_{zw}^{\alpha})\big]\Big\}\nonumber\\
    &=  \mathbb{E}_{\hat{\pi}(w^*,d^*,z^*)}\Big\{\mathbb{I}(z=z^*)\big[1 - \sum_{w=1}^V\exp(M_{zw}^{\alpha})\big]\Big\}\nonumber\\
    &= -\zeta_z^{\alpha}(\mathbf{M}(s))\mathbb{E}_{\hat{\pi}(w^*,d^*,z^*)}\mathbb{I}(z=z^*),\forall z\label{equ.proof ML continuous 1}\\
    \frac{d}{ds}\zeta_{d}^{\beta}(\mathbf{M}(s)) &= \nabla \zeta_{d}^{\beta}(\mathbf{M}(s)) \cdot g(\mathbf{M}(s))\nonumber\\
    &= \sum_{z=1}^K \exp(M_{zd}^{\beta}) \mathbb{E}_{\hat{\pi}(w^*,d^*,z^*)}\Big\{\mathbb{I}(d=d^*)\big[\mathbb{I}(z=z^*)\exp(-M_{zd}^{\beta}) - 1\big]\Big\}\nonumber\\
    &= \sum_{z=1}^K \mathbb{E}_{\hat{\pi}(w^*,d^*,z^*)}\Big\{\mathbb{I}(d=d^*)\big[\mathbb{I}(z=z^*) - \exp(M_{zd}^{\beta})\big]\Big\}\nonumber\\
    &= \mathbb{E}_{\hat{\pi}(w^*,d^*,z^*)}\Big\{\mathbb{I}(d=d^*)\big[1 - \sum_{z=1}^K\exp(M_{zd}^{\beta})\big]\Big\}\nonumber\\
    &= -\zeta_d^{\beta}(\mathbf{M}(s))\mathbb{E}_{\pi(w^*,d^*)}\mathbb{I}(d=d^*),\forall d.
\label{equ.proof ML continuous 2}
\end{align}
In another word, their absolute values strictly monotonically decrease over time,
\begin{align*}
    \frac{d}{ds}|\zeta_{z}^{\alpha}(\mathbf{M}(s))| &= -|\zeta_z^{\alpha}(\mathbf{M}(s))|\mathbb{E}_{\hat{\pi}(w^*,d^*,z^*)}\mathbb{I}(z=z^*)< 0,\forall z,\nonumber\\
    \frac{d}{ds}|\zeta_{d}^{\beta}(\mathbf{M}(s))| &= -|\zeta_d^{\beta}(\mathbf{M}(s))|\mathbb{E}_{\hat{\pi}(w^*,d^*,z^*)}\mathbb{I}(d=d^*)< 0,\forall d,
\end{align*}
until $\mathbf{M}$ reaches the manifold $H$.
Therefore regardless of $\mathbf{M}(0)$, $\mathbf{M}(s)$ monotonically converges to the manifold $H$.
Informally speaking, the more distant the $\mathbf{M}$ is from the manifold, the faster the the distance decreases, e.g., $\frac{d}{ds}|\zeta_{z}^{\alpha}(\mathbf{M}(s))|$. The monotonic decreasing property has two consequences. One is $\forall s\geq 0, \mathbf{M}(s)\in C$ by the definitions of the set, Eq.~(\ref{equ.proof ML close set def}); at the time limit, $\mathbf{M}$ converges to the normalization manifold, $F(C) \subseteq H$. The other is that, if at time $s'$, $\mathbf{M}(s')\in H$, then the trajectory will not escape from the manifold, $\forall s\geq s', \mathbf{M}(s)\in H$.

Another relation about the sets that we now prove is that, $S(C)\subseteq H$. From Eq.~(\ref{equ.proof ML ML-ODE}), we can solve the stationary points $g(\mathbf{M^*}) = \mathbf{0}$ implicitly:
\begin{align}
    g(M_{zw}^{\alpha,*}) &= 0 \leftrightarrow M_{zw}^{\alpha,*} = \log \frac{\mathbb{E}_{\hat{\pi}(w^*,d^*,z^*)}\mathbb{I}(w=w^*,z=z^*)}{\mathbb{E}_{\pi(w^*,d^*,z^*)}\mathbb{I}(z=z^*)},\forall z,w,\nonumber\\
    g(M_{zd}^{\beta,*}) &= 0 \leftrightarrow M_{zd}^{\beta,*} = \log \frac{\mathbb{E}_{\hat{\pi}(w^*,d^*,z^*)}\mathbb{I}(d=d^*,z=z^*)}{\mathbb{E}_{\pi(w^*,d^*)}\mathbb{I}(d=d^*)},\forall z,d.
\end{align}
Simple algebra yields that $\sum_{w=1}^V \exp(M_{zw}^{\alpha,*}) = 1, \forall z, \sum_{z=1}^K \exp(M_{zd}^{\beta,*})=1,\forall d$, showing $S(C)\subseteq H$.

\subsubsection{SpikeLDA}
In SpikeLDA algorithm, the ML-ODE is
\begin{align}
    \frac{d}{ds}M_{zw}^{\alpha} &= g^{\alpha}(M_{zw}^{\alpha}),\forall z,w,\nonumber\\
    \frac{d}{ds}M_{zd}^{\beta} &= g^{\beta}_{map}(M_{zd}^{\beta})\triangleq \mathbb{E}_{\hat{\pi}(w^*,d^*,z^*)}\Big\{\mathbb{I}(d=d^*)\big[\big(\mathbb{I}(z=z^*)+\frac{\lambda_z-1}{N_d}\big)\exp(-M_{zd}^{\beta}) - \frac{1}{\kappa}-\frac{1}{N_d}\big]\Big\},\forall z,d,\label{equ.proof MAP ML-ODE}
\end{align}
where $g^{\alpha}(M_{zw}^{\alpha})$ is defined in Eq.~(\ref{equ.proof ML ML-ODE}).

As the first result, $\mathbf{M}(s)$ monotonically converges to the normalization manifold $H$. To see this, the temporal evolutions of deviations from the manifold are:
\begin{align*}
    \frac{d}{ds}\xi_{z}^{\alpha}(\mathbf{M}(s)) = -\xi_z^{\alpha}(\mathbf{M}(s))\mathbb{E}_{\hat{\pi}(w^*,d^*,z^*)}\mathbb{I}(z=z^*),\forall z
\end{align*}
as derived in Eq.~(\ref{equ.proof ML continuous 1}), and
\begin{align*}
    \frac{d}{ds}\xi_{d}^{\beta}(\mathbf{M}(s)) &= \nabla \xi_{d}^{\beta}(\mathbf{M}(s)) \cdot g(\mathbf{M}(s))\nonumber\\
    &= \sum_{z=1}^K \exp(M_{zd}^{\beta}) \mathbb{E}_{\pi(w^*,d^*,z^*)}\bigg\{\mathbb{I}(d=d^*)\Big[\big(\mathbb{I}({z=z^*})+\frac{\lambda_z - 1}{N_d}\big)\exp(-M_{zd}^{\beta}) -\frac{1}{\kappa}- \frac{1}{N_d})\Big]\bigg\}\nonumber\\
    &= \sum_{z=1}^K \mathbb{E}_{\pi(w^*,d^*,z^*)}\bigg\{\mathbb{I}(d=d^*)\Big[\mathbb{I}({z=z^*})+\frac{\lambda_z - 1}{N_d} -\big(\frac{1}{\kappa}+ \frac{1}{N_d})\exp(M_{zd}^{\beta}) \Big]\bigg\}\nonumber\\
    &= \mathbb{E}_{\pi(w^*,d^*,z^*)}\bigg\{\mathbb{I}(d=d^*)\Big[1+\frac{\sum_{z=1}^K (\lambda_z - 1)}{N_d} -\big(\frac{1}{\kappa}+ \frac{1}{N_d})\sum_{z=1}^K \exp(M_{zd}^{\beta}) \Big]\bigg\}\nonumber\\
    &= \mathbb{E}_{\pi(w^*,d^*,z^*)}\bigg\{\mathbb{I}(d=d^*)\Big[1+\frac{\kappa}{N_d} -\big(\frac{1}{\kappa}+ \frac{1}{N_d})\sum_{z=1}^K \exp(M_{zd}^{\beta}) \Big]\bigg\}\nonumber\\
    &= -\frac{N_d+\kappa}{\kappa N_d}\pi(d)\zeta_d^{\beta}(\mathbf{M}(s)),\forall d.
\end{align*}
Then the monotonically convergence to $H$ follows, analogue to the arguments in proof \ref{proof.ML dynamics}.

Next, we prove $S(C)\subseteq H$. From Eq.~(\ref{equ.proof MAP ML-ODE}), we can solve the stationary points $g(\mathbf{W^*}) = \mathbf{0}$ implicitly:
\begin{align*}
    g(M_{zw}^{\alpha,*}) &= 0 \leftrightarrow M_{zw}^{\alpha,*} = \log \frac{\mathbb{E}_{\hat{\pi}(w^*,d^*,z^*)}\mathbb{I}(w=w^*,z=z^*)}{\mathbb{E}_{\hat{\pi}(w^*,d^*)}\mathbb{I}(z=z^*)},\forall z,w,\nonumber\\
    g(M_{zd}^{\beta,*}) &= 0 \leftrightarrow M_{zd}^{\beta,*} = \log \frac{\mathbb{E}_{\hat{\pi}(w^*,d^*,z^*)}\Big[\mathbb{I}(d=d^*,z=z^*)+\mathbb{I}(d=d^*)\frac{\lambda_z - 1}{N_d}\Big]}{\mathbb{E}_{\pi(w^*,d^*)}\bigg[\mathbb{I}(d=d^*)\Big(\frac{1}{\kappa}+\frac{1}{N_d}\Big)\bigg]},\forall z,d.
\end{align*}
Simple algebra yields that $\sum_{w=1}^V \exp(M_{zw}^{\alpha,*}) = 1, \forall z, \sum_{z=1}^K \exp(M_{zd}^{\beta,*})=\kappa,\forall d$, showing $S(C)\subseteq H$.

Following the same proof as in proof~\ref{proof.ML dynamics}, once $\mathbf{M}$ enters the manifold $H$ at time $s'$, it stays on the manifold for $s\geq s'$.

\subsubsection{semi-SpikeLDA}
In semi-SpikeLDA, we only consider $\mathbf{M}^{\alpha}$ for optimization.
The proof is the same as that in the SpikePLSI algorithm.\qed
\end{num_proof}

\subsubsection{Convergence to the local maximum on the manifold}
\begin{num_proof}\label{proof.convergence to maximum}
As proved above, once $\mathbf{M}$ enters the manifold $H$ at time $s'$, it stays on the manifold for $s\geq s'$.
In this part, we prove that $\mathbf{M}$ following the ML-ODEs converges to the local maximum of the corresponding conceived objective functions, for SpikePLSI, SpikeLDA and semi-SpikeLDA respectively.

\subsubsection{SpikePLSI} We denote the objective function as $\mathbb{L}^{plsi}(\mathbf{M})$ for brevity:
$$\mathbb{L}(\mathbf{M})^{plsi}\triangleq\mathbb{E}_{\pi(w,d)}\log p(x_w^{\alpha}=1|x_d^{\beta}=1;\mathbf{M}^{\alpha}, \mathbf{M}^{\beta}).$$
Given $\mathbf{M}(s)\in H$, $\mathbb{L}^{plsi}$ strictly monotonically increases over time until $\mathbf{M}$ reaches a stationary point $g(\mathbf{M}) = \mathbf{0}$:
\begin{align}
    \frac{d}{ds} \mathbb{L}^{plsi}(\mathbf{M}(s))=& \nabla \mathbb{L}^{plsi}(\mathbf{M}) \cdot g(\mathbf{M})\nonumber\\
    =& \sum_{z=1}^K\sum_{w=1}^V \exp(M_{zw}^{\alpha})g(M_{zw}^{\alpha})^2 + \sum_{d=1}^D\sum_{z=1}^K \exp(M_{zd}^{\beta})g(M_{zd}^{\beta})^2\geq 0. \label{equ.proof ML continuous 3}
\end{align}
In another word, $F(C)\subseteq S(C)$. Combining the results above, $F(C)\subseteq S(C)\cap H = S(C)$. The second equality in Eq.~(\ref{equ.proof ML continuous 3}) is from
\begin{align}
     \frac{\partial}{\partial M_{zw}^{\alpha}}\mathbb{L}^{plsi}
    =&\frac{\partial}{\partial M_{zw}^{\alpha}} \mathbb{E}_{\pi(w^*,d^*)}\log p(x_{w^*}^{\alpha}=1|x_{d^*}^{\beta}=1;\mathbf{M}^{\alpha}, \mathbf{M}^{\beta})\nonumber\\
    =& \mathbb{E}_{\hat{\pi}(w^{*},d^{*},z^{*})} \frac{\partial}{\partial M_{zw}^{\alpha}}  \log \sum_{z=1}^K \exp(M_{zw^{*}}^{\alpha}+M_{zd^{*}}^{\beta} - \log A(M_{z\cdot}^{\alpha},M_{\cdot d^{*}}^{\beta}))\nonumber\\
    =& \mathbb{E}_{\pi(w^{*},d^{*})} \bigg\{\frac{\exp(M_{zw^*}^{\alpha}+M_{zd^{*}}^{\beta} - \log A(M_{z\cdot}^{\alpha},M_{\cdot d^{*}}^{\beta}))}{\sum_{z=1}^K \exp(M_{zw^*}^{\alpha}+M_{z d^{*}}^{\beta} - \log A(M_{z\cdot}^{\alpha},M_{\cdot d^{*}}^{\beta}))}\Big[\mathbb{I}(w=w^*) - \exp(M_{zw}^{\alpha})\Big]\bigg\}\nonumber\\
    =& \mathbb{E}_{\pi(w^{*},d^{*})} \bigg\{p(h_z=1|x_{w^*}^{\alpha}=1,x_{d^*}^{\beta}=1)\Big[\mathbb{I}(w=w^*) - \exp(M_{zw}^{\alpha})\Big] \bigg\}\nonumber\\
    =& \mathbb{E}_{\hat{\pi}(w^{*},d^{*},z^{*})} \bigg\{\mathbb{I}(w=w^*,z=z^*) - \mathbb{I}(z=z^*)\exp(M_{zw}^{\alpha}) \bigg\}\nonumber\\
    =& \exp(M_{zw}^{\alpha})g(M_{zw}^{\alpha}),\forall z,w,\label{equ.proof ML continuous 4}\\
     \frac{\partial}{\partial M_{zd}^{\beta}}\mathbb{L}^{plsi}
    =&\frac{\partial}{\partial M_{zd}^{\beta}} \mathbb{E}_{\pi(w^*,d^*)}\log p(x_{w^*}^{\alpha}=1|x_{d^*}^{\beta}=1;\mathbf{M}^{\alpha}, \mathbf{M}^{\beta})\nonumber\\
    =& \mathbb{E}_{\pi(w^*,d^*)} \frac{\partial}{\partial M_{zd}^{\beta}}  \log \sum_{z=1}^K \exp(M_{zw^*}^{\alpha}+M_{zd^*}^{\beta} - \log A(M_{z\cdot}^{\alpha},M_{\cdot d^*}^{\beta}))\nonumber\\
    =& \mathbb{E}_{\pi(w^*,d^*)}\bigg\{\mathbb{I}(d=d^*)\Big.[\frac{\exp(M_{zw^*}^{\alpha}+M_{zd}^{\beta} - \log A(M_{z\cdot}^{\alpha},M_{\cdot d^{*}}^{\beta})}{\sum_{z=1}^K \exp(M_{zw^*}^{\alpha}+M_{zd}^{\beta} - \log A(M_{z\cdot}^{\alpha},M_{\cdot d^{*}}^{\beta})}- \exp(M_{zd}^{\beta})\Big]\bigg\}\nonumber\\
    =& \mathbb{E}_{\pi(w^*,d^*)}\bigg\{\mathbb{I}(d=d^*)\Big[(p(h_z=1|x_{w^*}^{\alpha}=1,x_d^{\beta} = 1) - \exp(M_{zd}^{\beta})\Big]\bigg\}\nonumber\\
    =& \mathbb{E}_{\hat{\pi}(w^*,d^*,z^*)}\bigg\{\mathbb{I}(d=d^*)\Big[\mathbb{I}(z=z^*) - \exp(M_{zd}^{\beta})\Big]\bigg\}\nonumber\\
    =& \exp(M_{zd}^{\beta})g(M_{zd}^{\beta}),\forall z,d.\label{equ.proof ML continuous 5}
\end{align}
In another word, $g(\mathbf{M})$ is the natural gradient, where the metric is the expected Fisher information matrix of a Poisson likelihood \cite{Patterson2013}. $S(C)$ is not only the set of stationary points, but also the set of all the critical points of the objective function $\mathbb{L}^{plsi}$ on the constrained manifold. It is obvious that $F(C) \supseteq S(C)$:
\begin{align}
    \lim_{s\rightarrow\infty} \mathbf{M}(s) = \mathbf{M}(0), \text{if }g(\mathbf{M}(0)) = \mathbf{0}.
\end{align}
So $F(C) = S(C)$. The positive limit set $F(C)$ equals the set of critical points to the constrained optimization problem (\ref{equ.ML SNN problem}). With probability one, $\mathbf{M}(s)$ converges to a local maximum as $s\rightarrow \infty$.

So if $\sum_t\eta_t = \infty,\sum_t \eta_t^2 <\infty$, all sequence under the stochastic update rule Eq.~(\ref{equ.ML dicrete dynamics}):
\begin{align}
    M_{zw}^{\alpha} &= M_{zw}^{\alpha} + \eta_t h_z\left.(x_w^{\alpha}\exp(-M_{zw}^{\alpha})-1\right.),\forall z,w,\nonumber\\
    M_{zd}^{\beta} &= M_{zd}^{\beta} + \eta_t x_d^{\beta}\left.(h_z\exp(-M_{zd}^{\beta})-1\right.),\forall z,d
\end{align}
converge, and the set of all stable convergence points is the same as the set of local maxima of the optimization problem (\ref{equ.ML SNN problem}).
The stochastic update rule can solve the optimization problem.\qed

\subsubsection{SpikeLDA}
We denote the objective function as $\mathbb{L}^{lda}(\mathbf{M})$ for brevity
$$\mathbb{L}^{lda}(\mathbf{M})\triangleq\mathbb{E}_{\pi(w,d)}\Big[\log p(x_w^{\alpha}=1|x_d^{\beta}=1;\mathbf{M}^{\alpha}, \mathbf{M}^{\beta}) + \frac{1}{N_d}\log p(\mathbf{M}_{\cdot d}^{\beta};\boldsymbol\lambda)\Big].$$
Given $\mathbf{M}(s)\in H$, $\mathbb{L}^{lda}$ monotonically increase over time until $\mathbf{M}$ reaches a stationary point $g(\mathbf{M}) = \mathbf{0}$:
\begin{align}
    \frac{d}{ds} \mathbb{L}^{lda}(\mathbf{M}(s))=& \nabla \mathbb{L}^{lda}(\mathbf{M}) \cdot g(\mathbf{M})\nonumber\\
    =& \sum_{z=1}^K\sum_{w=1}^V \exp(M_{zw}^{\alpha})g^{\alpha}(M_{zw}^{\alpha})^2 + \sum_{d=1}^D \sum_{z=1}^K \exp(M_{zd}^{\beta})g^{\beta}_{map}(M_{zd}^{\beta})^2\nonumber\\
    \geq & 0.\label{equ.proof MAP continuous 1}
\end{align}
In another word, $F(C)\subseteq S(C)$. Combining the results above, $F(C)\subseteq S(C)\cap H = S(C)$. The second equality in Eq.~(\ref{equ.proof MAP continuous 1}) is from
\begin{align*}
     \frac{\partial}{\partial M_{zw}^{\alpha}}\mathbb{L}^{lda}
    =&\frac{\partial}{\partial M_{zw}^{\alpha}} \mathbb{E}_{\pi(w,d)}\Big[\log p(x_w^{\alpha}=1|x_d^{\beta}=1;\mathbf{M}^{\alpha}, \mathbf{M}^{\beta}) + \log p(\mathbf{M}^{\beta}_{\cdot,d};\boldsymbol\lambda)\Big]\nonumber\\
    =& \frac{\partial}{\partial M_{zw}^{\alpha}}  \mathbb{E}_{\pi(w,d)}\log p(x_w^{\alpha}=1|x_d^{\beta}=1;\mathbf{M}^{\alpha}, \mathbf{M}^{\beta})\\
    =& \exp(M_{zw}^{\alpha})g(M_{zw}^{\alpha}),\forall z,w,
\end{align*}
where the last equality is from Eq.~(\ref{equ.proof ML continuous 4}). And

\begin{align*}
    & \frac{\partial}{\partial M_{zd}^{\beta}}\mathbb{L}^{lda}\nonumber\\
    =&\frac{\partial}{\partial M_{zd}^{\beta}} \mathbb{E}_{\pi(w,d)}\Big[\log p(x_w^{\alpha}=1|x_d^{\beta}=1;\mathbf{M}^{\alpha}, \mathbf{M}^{\beta}) + \frac{1}{N_d}\log p(\mathbf{M}^{\beta}_{\cdot d};\boldsymbol\lambda)\Big]\nonumber\\
    =& \mathbb{E}_{\pi(w,d^*)}\Bigg\{\mathbb{I}(d=d^*) \bigg\{\frac{\partial}{\partial M_{zd}^{\beta}}  \log \sum_{z=1}^K \exp(M_{zw}^{\alpha}+M_{zd}^{\beta} - \log A(M_{z\cdot}^{\alpha},M_{\cdot d}^{\beta}))+ \frac{1}{N_d}\frac{\partial}{\partial M_{zd}^{\beta}}\log \exp\Big[(\lambda_z - 1) M_{zd}^{\beta})\exp(-e^{M_{zd}^{\beta}})\Big]\bigg\}\Bigg\}\nonumber\\
    =& \mathbb{E}_{\pi(w,d^*)}\Bigg\{\mathbb{I}(d=d^*) \bigg\{\Big[\frac{\exp(M_{zw}^{\alpha}+M_{zd}^{\beta} - \log A(M_{z\cdot}^{\alpha},M_{\cdot d}^{\beta}))}{\sum_{z=1}^K \exp(M_{zw}^{\alpha}+M_{zd}^{\beta} - \log A(M_{z\cdot}^{\alpha},M_{\cdot d}^{\beta}))}- \frac{\exp(M_{zd}^{\beta})}{\kappa} \Big]+ \frac{1}{N_d}\Big[(\lambda_z - 1) - e^{M_{zd}^{\beta}}\Big]\bigg\}\Bigg\}\nonumber\\
    =& \mathbb{E}_{\hat{\pi}(w,d^*,z^*)}\Bigg\{\mathbb{I}(d=d^*) \bigg\{\Big[\mathbb{I}(z=z^*)- \frac{\exp(M_{zd}^{\beta})}{\kappa} \Big]+ \frac{1}{N_d}\Big[(\lambda_z - 1) - e^{M_{zd}^{\beta}}\Big]\bigg\}\Bigg\}\nonumber\\
    =& \mathbb{E}_{\hat{\pi}(w,d^*,z^*)}\Bigg\{\mathbb{I}(d=d^*) \bigg\{\Big[\mathbb{I}(z=z^*)+\frac{\lambda_z-1}{N_d}\Big]+ \Big[-\frac{\exp(M_{zd}^{\beta})}{\kappa} - \frac{e^{M_{zd}^{\beta}}}{N_d}\Big]\bigg\}\Bigg\}\nonumber\\
     =& \exp(M_{zd}^{\beta})\mathbb{E}_{\hat{\pi}(w,d^*,z^*)}\Bigg\{\mathbb{I}(d=d^*) \bigg\{\Big[\mathbb{I}(z=z^*)+\frac{\lambda_z-1}{N_d}\Big]\exp(-M_{zd}^{\beta}) - \Big(\frac{1}{\kappa} + \frac{1}{N_d}\Big)\bigg\}\Bigg\}\nonumber\\
    =& \exp(M_{zd}^{\beta})g^{\beta}_{map}(M_{zd}^{\beta}),\forall z,d.
\end{align*}
In another word, $g(\mathbf{M})$ is the natural gradient, where the metric is the expected Fisher information matrix of a Poisson likelihood \cite{Patterson2013}. $S(C)$ is not only the set of stationary points, but also the set of all the critical points of the objective function $\mathbb{L}^{lda}$ on the constrained manifold. It is obviously $F(C) \supseteq S(C)$:
\begin{align*}
    \lim_{s\rightarrow\infty} \mathbf{M}(s) = \mathbf{M}(0), \text{if }g(\mathbf{M}(0)) = \mathbf{0}.
\end{align*}
So $F(C) = S(C)$. The positive limit set $F(C)$ is the set of all the critical points of the constrained optimization problem (\ref{equ.MAP SNN problem}). With probability one, $\mathbf{M}(s)$ converges to a local maximum as $s\rightarrow \infty$.

So if $\sum_t\eta_t = \infty,\sum_t \eta_t^2 <\infty$, all sequences under the stochastic update rule Eq.~(\ref{equ.MAP discrete dynamics}):
\begin{align*}
    M_{zw}^{\alpha} &= M_{zw}^{\alpha} + \eta_t h_z\big(x_w^{\alpha}\exp(-M_{zw}^{\alpha})-1\big)\forall w,z,\\
    M_{zd}^{\beta} &= M_{zd}^{\beta} + \eta_t x_d^{\beta}\Big[(h_z+\frac{\lambda-1}{N_d})\exp(-M_{zd}^{\beta})- \frac{1}{\kappa} - \frac{1}{N_d}\Big],\forall d,z
\end{align*}
converge, and the set of all stable convergence points is the same as the set of local maxima of the optimization problem (\ref{equ.MAP SNN problem}).
The stochastic update rule can solve the optimization problem.
\subsubsection{semi-SpikeLDA}
We denote the objective function as $\mathbb{L}^{semi}(\mathbf{M})$ for brevity: $$
\mathbb{L}^{semi}(\mathbf{M}^{\alpha})\triangleq\mathbb{E}_{\pi(d)}\sum_{w=1}^V\sum_{z=1}^K\mathbb{E}_{\mathbf{z}_d|\mathbf{w}_d;\mathbf{M}^{\alpha},\mathbf{M}^{\beta}}[C_{d,z,w}]\log p(x_w^{\alpha}=1|h_z=1;\mathbf{M}^{\alpha}).$$
Given $\mathbf{M}^{\alpha}(s)\in H$, $\mathbb{L}^{semi}$ monotonically increase over time until $\mathbf{M}^{\alpha}$ reaches a stationary point $g(\mathbf{M}^{\alpha}) = \mathbf{0}$:
\begin{align}
    \frac{d}{dt} \mathbb{L}^{semi}(\mathbf{M}^{\alpha}(t))= \nabla \mathbb{L}^{semi}(\mathbf{M}^{\alpha}) \cdot g(\mathbf{M}^{\alpha})
    =\sum_{z=1}^K\sum_{w=1}^V \exp(M_{zw}^{\alpha})g^{\alpha}(M_{zw}^{\alpha})^2
    \geq  0.\label{equ.proof semi continuous 1}
\end{align}
In another word, $F(C)\subseteq S(C)$. Combining the results above, $F(C)\subseteq S(C)\cap H = S(C)$. The second equality in Eq.~(\ref{equ.proof semi continuous 1}) is from
\begin{align*}
     \frac{\partial}{\partial M_{zw}^{\alpha}}\mathbb{L}^{semi}
    =&\frac{\partial}{\partial M_{zw}^{\alpha}} \mathbb{E}_{\pi(d)}\sum_{w^*=1}^V\sum_{z^*=1}^K\mathbb{E}_{\mathbf{z}_d|\mathbf{w}_d;\mathbf{M}^{\alpha},\mathbf{M}^{\beta}}[C_{d,z^*,w^*}]\log p(x_{w^*}^{\alpha}=1|h_{z^*}^{\beta}=1;\mathbf{M}^{\alpha})\nonumber\\
    =&\frac{\partial}{\partial M_{zw}^{\alpha}} \mathbb{E}_{\pi(d)}\sum_{w^*=1}^V\sum_{z^*=1}^K\mathbb{E}_{\mathbf{z}_d|\mathbf{w}_d;\mathbf{M}^{\alpha},\mathbf{M}^{\beta}}[C_{d,z^*,w^*}]\log \exp(M_{z^*w^*}^{\alpha} - \log(\sum_{w'}\exp M_{z^*w'}))\nonumber\\
    =& \mathbb{E}_{\pi(d)}\sum_{w^*=1}^V\sum_{z^*=1}^K\mathbb{E}_{\mathbf{z}_d|\mathbf{w}_d;\mathbf{M}^{\alpha},\mathbf{M}^{\beta}}[C_{d,z^*,w^*}]\Big[\mathbb{I}(w=w^*,z=z^*)-\mathbb{I}(z=z^*)\exp(M_{zw}^{\alpha})\Big]\nonumber\\
    =& \mathbb{E}_{\pi(d)}\mathbb{E}_{\mathbf{z}_d|\mathbf{w}_d;\mathbf{M}^{\alpha},\mathbf{M}^{\beta}}\Big[C_{d,z,w}-C_{d,z}\exp(M_{zw}^{\alpha})\Big]\nonumber\\
    =& \exp(M_{zw}^{\alpha})g(M_{zw}^{\alpha}),~~~\forall z,w,
\end{align*}
where the last equality is from Eq.~(\ref{equ.proof ML continuous 4}).
In another word, $g(\mathbf{M})$ is the natural gradient, where the metric is the expected Fisher information matrix of a Poisson likelihood \cite{Patterson2013}. $S(C)$ is not only the set of stationary points, but also the set of all the critical points of the objective function $\mathbb{L}^{semi}$ on the constrained manifold. It is obviously $F(C) \supseteq S(C)$:
\begin{align*}
    \lim_{s\rightarrow\infty} \mathbf{M}^{\alpha}(s) = \mathbf{M}^{\alpha}(0), \text{if }g(\mathbf{M}^{\alpha}(0)) = \mathbf{0}.
\end{align*}
So $F(C) = S(C)$. The positive limit set $F(C)$ is the set of all the critical points of the constrained optimization problem (\ref{equ.MAP SNN problem}). With probability one, $\mathbf{M}(s)$ converges to a local maximum as $s\rightarrow \infty$.

So if $\sum_t\eta_t = \infty,\sum_t \eta_t^2 <\infty$, all sequences under the stochastic update rule Eq.~(\ref{equ.semi-likelihood discrete dynamics}):
\begin{align*}
    M_{zw}^{\alpha} &\leftarrow M_{zw}^{\alpha} + \eta_t \frac{1}{|\hat{D}|T}\Big[\hat{N}_{z,w}\exp(-M_{zw}^{\alpha}) - \hat{N}_{z}\Big],\forall w,z
\end{align*}
converge, and the set of all stable convergence points is the same as the set of local maxima of the optimization problem (\ref{equ.semi SNN problem}).
The stochastic update rule can solve the optimization problem.
\qed
\end{num_proof}

\section{Network pruning}\label{sec.pruning}
Due to the limitation of hardware technology, some NMSs constrain the number of hardware synapses connecting to a neuron.
For instance, TruthNorth~\cite{Merolla2014} limits the fan-in to a neuron to not greater than $256$.
To accommodate such hardware constrain, we propose an pruned extension of the ed-SpikeLDA, where the number of connections to a topic neuron is not greater than $256$ as required.

The pruning method has three steps.
First, we pre-train the ed-SpikeLDA for several iterations on a GPC.
After the pre-training, the top-200 words for each topic are identified.
Second, we prune the pre-trained network by tying some model parameters (synapses) before deploying on the NMS.
For a topic $z$, each word synapse $M_{wz}^{\alpha}$ of the top-200 words enjoys an unique adaptive hardware synapse; and
the other words synapses share only one hardware synapse fixed to $\log(\frac{p}{V-200})$, where $p$ is the sum of $\phi_{zw}$  of the $V-200$ less-relevant words calculated according to Tab.~\ref{tlb.Parameter relation}.
As the document synapses are local parameters,
we only keep the $\mathbf{M}_{\cdot d}$ of the $50$ most frequent documents on the NMS, and store the others in an external memory.
After the second step, each topic neuron is connected to only $251$ hardware synapses.
Third, ed-SpikeLDA continues its training on the NMS.
\subsection{Impacts of network pruning}
The pruned network would output a high perplexity because the $V-200$ less-relevant words equally share their weights in the topic distribution.
However, if we examine the learnt latent representations, we find that the network still outputs useful representations for discrimination (see Tab.~\ref{tlb.pruned classification}).
The discriminative experiment is the same as that in the main text.
\begin{table}[h!]
    \vspace{-.2cm}
    \centering
    \resizebox{0.33\textwidth}{!}{%
    \begin{tabular}{|c|c c |}\hline
       & Binary &  Mutli-class\\ \hline
    CGS& $72.8\pm 3.8$& $63.5\pm 0.2$ \\
    ed-SpikeLDA&  $72.5\pm 2.4$& $57.6\pm 3.3$\\
    ed-SpikeLDA(p-15)& $73.0\pm0.1$ &$60.0\pm1.3$\\
    ed-SpikeLDA(p-20)& $73.8\pm 1.2$ &$62.2\pm 1.4$\\\hline
    \end{tabular}%
    }
    \vspace{-.1cm}
    \caption{Classification on 20NG, where the $p$-$x$ means pre-training takes $x$ iterations.}
    \label{tlb.pruned classification}
\vspace{-0cm}
\end{table}

\end{document}